%% file: neurips_2025.tex
\documentclass{article}

\PassOptionsToPackage{numbers, compress}{natbib}

\usepackage[final]{neurips_2025}

\usepackage[utf8]{inputenc} %
\usepackage[T1]{fontenc}    %
\usepackage{hyperref}       %
\usepackage{url}            %
\usepackage{booktabs}       %
\usepackage{amsfonts}       %
\usepackage{nicefrac}       %
\usepackage{microtype}      %
\usepackage{xcolor}         %

\usepackage{inconsolata}

\usepackage{custom}
\input{macros}

\title{Characterizing the Expressivity of Fixed-Precision Transformer Language Models}

\author{
Jiaoda Li%
~\;~\;~Ryan Cotterell\\
\texttt{\{\href{mailto:jiaoda.li@inf.ethz.ch}{jiaoda.li}, \href{mailto:ryan.cotterell@inf.ethz.ch}{ryan.cotterell}\}@inf.ethz.ch}\\
    {%
\setlength{\fboxsep}{2.5pt}%
\setlength{\fboxrule}{2.5pt}%
    \includegraphics[width=.15\linewidth]{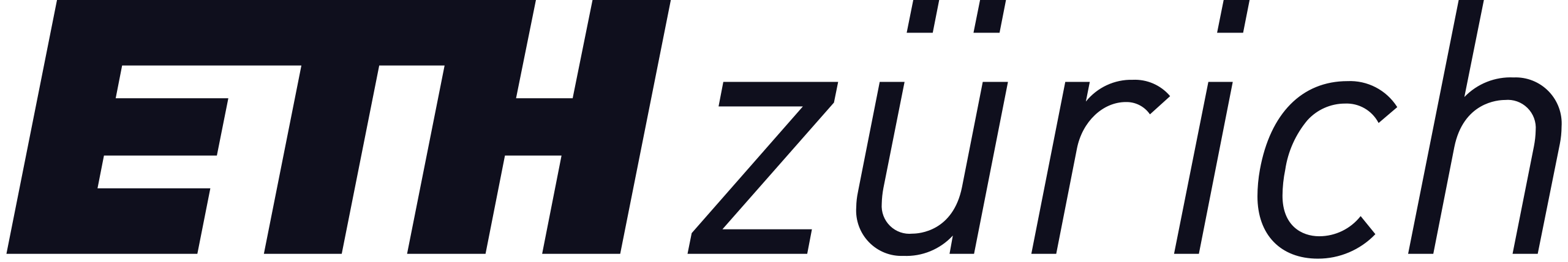}
}}

\begin{document}

\maketitle

\begin{abstract}
\label{sec:abstract}
    Transformer-based language models (LMs) have achieved widespread empirical success, but their theoretical expressive power remains only partially understood. 
    In this work, we analyze a restricted idealization of fixed-precision transformers with strict future masking, soft attention, and no positional encodings. 
    We establish that this class of models is exactly as expressive as a specific fragment of linear temporal logic that contains only a single temporal operator: the \texttt{past} operator. 
    We further connect this fragment to established classes in formal language theory, automata theory, and algebra, yielding a unified framework for understanding transformer expressivity under this idealization. 
    Finally, we present empirical results that align closely with our theory: transformers trained on languages within their characterized expressive capacity generalize reliably across sequence lengths, while they consistently fail to generalize on languages beyond it.\footnote{Code available at 
\href{https://github.com/rycolab/expressivity-of-fixed-precision-transformers}
{GitHub repository}.}
\end{abstract}

\section{Introduction}
\label{sec:intro}
Transformer-based language models (LMs) have demonstrated remarkable empirical success \citep{NIPS2017_3f5ee243, DBLP:journals/corr/abs-2303-08774, dubey2024llama3herdmodels} on a wide variety of natural language tasks \citep[][\textit{inter alia}]{NEURIPS2019_4496bf24,hendrycks2021measuring,srivastava2023beyond}.
This success has sparked growing interest in understanding the theoretical expressive power of transformers, i.e., what languages they can and cannot recognize, and, by extension, what tasks they can and cannot perform. 
A significant body of work approaches this question by relating transformers to well-established frameworks such as formal languages, logic, and circuit complexity \citep{hao-etal-2022-formal,NEURIPS2023_a48e5877,yang2024masked,10.1162/tacl_a_00663}.
To facilitate their theoretical analysis, theoreticians often propose idealizations of transformers. 
For instance, while practical implementations of transformers operate under fixed precision, e.g., single (32-bit) or half (16-bit) precision, many authors assume arbitrary \citep{pérez2018on,hao-etal-2022-formal,merrill-etal-2022-saturated} or length-dependent precision \citep{merrill-sabharwal-2023-parallelism,chiang2025transformers}.
Although such idealizations capture key aspects of transformers, they tend to overestimate their expressive power \citep{pérez2018on}.\looseness=-1

A recent step toward a more faithful theoretical understanding of the expressive power of transformers comes from \citet{yang2024masked}, who show that fixed-precision transformers with strict future masking and unique hard attention (UHA) are exactly as expressive as linear temporal logic $\ltl$, which includes four temporal operators: $\past$ (\texttt{past}), $\future$ (\texttt{future}), $\since$ (\texttt{since}), and $\until$ (\texttt{until}). 
However, UHA still deviates from the soft attention used in practice. To address this gap, \citet{yang2024countingliketransformerscompiling} analyze fixed-precision transformers with strict future masking and soft attention, an idealization that most closely reflects the models deployed in real-world applications. \citet{yang2024countingliketransformerscompiling} show that such models are upper bounded by C-RASP, a counting-based programming language, though a precise characterization of these models' expressivity remains open.

In this paper, we close this gap by providing an exact characterization of the expressive power of fixed-precision transformers with soft attention, strict masking, and no positional encodings (NoPE).
We show they are precisely characterized by $\ptl$, a restricted fragment of $\ltl$ that uses only the \texttt{past} operator ($\past$). We further demonstrate that $\ptl$ is equivalent in expressivity to partially ordered deterministic finite automata (PODFAs), which are characterized by $\gR$-trivial monoids and recognize left-deterministic polynomials.
These results offer a detailed and principled characterization of the expressive power of this idealization, delineating its strengths and limitations.
Crucially, our findings imply that many simple languages, e.g., bounded Dyck languages, which have been shown to be recognizable under more permissive idealizations, are beyond the reach of the models we study.
We also extend our theoretical results to transformer LMs, showing that their expressivity matches that of transformer recognizers. A visual overview of the theoretical landscape is provided in \cref{fig:roadmap}.

{
\hypersetup{hidelinks}
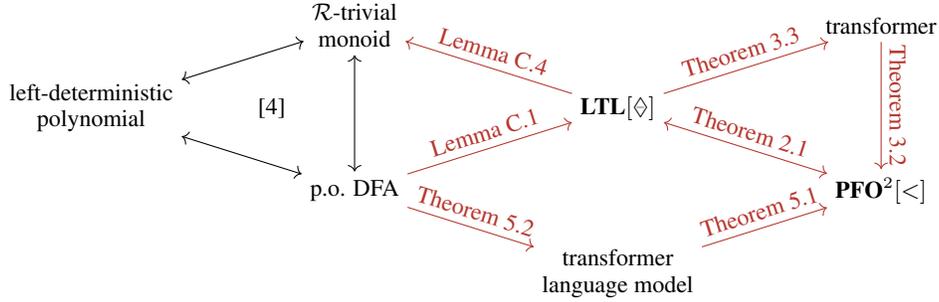
\begin{figure*}
    \centering
    \footnotesize
    \begin{tikzpicture}
        \node (ptl) at (0,0) {$\ptl$};
        \node (pfo) at (3.5,-1.1) {$\pfo$};
        \node[align=center] (poa) at (-3.5,-1.1) {PODFA};
        \node[align=center] (rtm) at (-3.5,1.1) {$\gR$-trivial\\monoid};
        \node[align=center] (ldp) at (-7,0) {left-deterministic\\polynomial};
        \node[align=center] (tlm) at (0,-2.2) {transformer\\language model};
        \node[align=center] (tr) at (3.5,1.1) {transformer};
        \draw[<->] (poa) -- (ldp);
        \draw[<->] (poa) -- (rtm);
        \draw[<->] (rtm) -- (ldp);
        \draw[<->,ETHRed] (pfo) -- node[sloped,above]{\cref{thm:pfo2ptl}} (ptl);
        \draw[->,ETHRed] (poa) -- node[sloped,above]{\cref{thm:po2tlm}} (tlm);
        \draw[->,ETHRed] (tlm) -- node[sloped,above]{\cref{thm:tlm2pfo}} (pfo);
        \draw[->,ETHRed] (ptl) -- node[sloped,above]{\cref{thm:ptl2transformer}} (tr);
        \draw[->,ETHRed] (tr) -- node[sloped,above]{\cref{thm:transformer2pfo}} (pfo);
        \draw[->,ETHRed] (poa) -- node[sloped,above]{\cref{lemma:po2ptl}} (ptl);
        \draw[->,ETHRed] (ptl) -- node[sloped,above]{\cref{lemma:ptl2rt}} (rtm);
        \node at (-4.6,0) {\citep{BRZOZOWSKI198032}};
    \end{tikzpicture}
    \caption{Roadmap of the paper. \textcolor{ETHRed}{Red arrows} indicate novel results. }
    \label{fig:roadmap}
\end{figure*}
}

To arrive at a compelling theory, it is essential to show that it faithfully reflects the behavior of models trained under standard machine learning paradigms. To this end, we provide empirical evidence using the length generalization framework, a widely used method for gauging neural network expressivity \citep{deletang2023neural, butoi2024trainingneuralnetworksrecognizers, huang2025a}. We construct a suite of languages spanning a fine-grained hierarchy of formal language classes. Our results (\cref{tab:results}) exhibit strong alignment between theory and practice: for all languages that transformers are predicted to recognize, the models generalize perfectly over lengths ($100\%$ accuracy); for languages beyond their theoretical capacity, they consistently make generalization errors, regardless of learning rates or random seeds.

\section{Background}
In this section, we present the necessary background knowledge that underpins our analysis.
\subsection{Strings and Languages}
An \defn{alphabet}, denoted as $\alphabet$, is a finite, non-empty set of symbols. 
A \defn{string} over $\alphabet$ is a \emph{finite} sequence of symbols drawn from $\alphabet$. 
The set of all strings over $\alphabet$ is denoted by its Kleene star $\kleene{\alphabet}$. 
A subset of $\kleene{\alphabet}$ is called a \defn{language}.
A \defn{regular expression} is a declarative way to describe a language, defined recursively as follows:\looseness=-1
\begin{itemize}%
    \item $\emptyset$ and each $\syma\in\alphabet$ are regular expressions;
    \item If $\alpha$ and $\beta$ are regular expressions, so are the union $\alpha+\beta$, concatenation $\alpha\beta$, the Kleene star $\alpha^*$, and complement $\setcomplement{\alpha}$.
\end{itemize}
A language is \defn{regular} if and only if it can be described by a regular expression \citep{Kleene1956}. A regular language is said to be \defn{star-free} if it can be described by a regular expression without the Kleene star \citep{mcnaughton1971counter}. As an example, $\kleene{\alphabet}$, the set of all strings over $\alphabet$, is star-free, as it can be described by $\setcomplement{\emptyset}$.
\looseness=-1

\subsection{$\ptl$}
Linear temporal logic $\ltl$ \citep{Kamp1968-KAMTLA-2} is a modal logic, with modalities referring to time. 
The full definition is given in \cref{app:tl}. In this paper, we define a fragment of $\ltl$---denoted as $\ptl$---that includes only one temporal operator $\past$ (\texttt{past}).
\defn{Formulas} in $\ptl$ are composed of \defn{atomic formulas} $\atom_\syma$ for every $\syma\in\alphabet$, Boolean connectives $\land,\lnot$, and a temporal operator $\past$. The disjunction $\lor$ is definable in terms of $\land$ and $\lnot$ as:
\begin{equation}
    \tlf_1 \lor \tlf_2 \defeq \lnot(\lnot \tlf_1 \land \lnot \tlf_2).
\end{equation}
When defined over strings, the formulas are interpreted with respect to a string $\str=\sym_1\cdots\sym_\length$ at a position $n\in\{1,\ldots,\length\}$. 
The semantics are defined inductively. Given a string $\str$, a position $n$, and a formula $\tlf$, we define $\str,n\models \tlf$ ($\str$ satisfies $\tlf$ at position $n$) as follows:
\begin{center}
    \begin{tabular}{lcl}
        $\str,n\models \atom_\syma$ & if and only if & the $n^{\text{th}}$ symbol of $\str$ is an $\syma$. \\
        $\str,n\models \tlf_1 \land \tlf_2$ & if and only if & $\str,n\models \tlf_1$ and $\str,n\models \tlf_2$. \\
        $\str,n\models \lnot \tlf$ & if and only if & $\str,n \not\models \tlf$. \\
        $\str,n\models \past \tlf$ & if and only if & there exists $m$ such that $m<n$ and $\str, m\models \tlf$. \\
    \end{tabular}
\end{center}
\begin{testexample}
    \begin{itemize}[leftmargin=*]
        \item $\syma\symb\symb,1 \models \atom_\syma$ as the first symbol is $\syma$.
        \item $\syma\symb\symb,3 \models \past\atom_\syma$ as an $\syma$ occurs before position 3.
    \end{itemize}
\end{testexample}
To say a string $\str$ of length $\length$ satisfies a formula $\tlf$, written as $\str\models\tlf$, we evaluate from position $\length + 1$ (the position after the final symbol):
\begin{equation}
    \str\models\tlf \quad \text{if and only if} \quad \str,\length+1\models \tlf.
\end{equation} 
The language defined by a formula $\tlf$ is the set $\lang(\tlf)$ of all strings that satisfy $\tlf$.
\begin{testexample}
    $\past\syma$ defines all the strings that contain at least one occurrence of $\syma$, i.e.,
    \begin{equation}
        \lang(\past\syma) = \kleene{\alphabet}\syma\kleene{\alphabet}.
    \end{equation}
\end{testexample}

\subsection{$\pfo$}
First-order logic $\fo$ can also be used to describe formal languages. 
In contrast to (linear) temporal logic, where formulas are interpreted at a single position, $\fo$ formulas may contain multiple position variables, making it better-suited to simulating mechanisms in transformers that involve more than one position, such as attention; see \cref{sec:transformer2pfo,app:transformer2pfo}. 
In this paper, we introduce a fragment of $\fo$ called $\pfo$---the past fragment of first-order logic with two variables. 
The atomic formulas of $\pfo$ consist of unary predicates $\atom_\syma$ for each $\syma\in\alphabet$ and a binary numerical predicate $<$ that can be used to determine the sequential order between variables. 
Formulas in $\pfo$ can have at most two distinct variables, $x$ and $y$. In addition to the usual Boolean connectives $\land$ and $\lnot$, there is a bounded existential quantifier $\exists y<x$, where $y$ is the variable bounded by the quantifier, and $x$ is the free variable. In case there is no free variable, the bounded existential quantifier behaves like an ordinary existential quantifier.
The formulas in $\pfo$ are constructed inductively as follows:
\begin{itemize}
    \item Every atomic formula is a formula;
    \item A Boolean combination of formulas is a formula if it does not contain more than two variables;
    \item If $\fof(x,y)$ is a formula with two variables $x,y$,  $\exists y<x: \fof(x,y)$ and $\exists x<y:\fof(x,y)$ are formulas;
    \item If $\fof(x)$ is a formula with one variable $x$, $\exists x<y: \fof(x)$ is a formula;
    \item If $\fof(x)$ is a formula with one variable $x$, $\exists x: \fof(x)$ is a formula.
\end{itemize}
The bounded universal quantifier $\forall y<x$ can be defined using $\exists y<x$ and $\lnot$. For instance:
\begin{equation}
    \forall y<x: \fof(x,y) \defeq \lnot \exists y<x: \lnot \fof(x,y).
\end{equation}
A formula with no free variables is called a \defn{sentence}. 
We write $\str\models\fof$ if $\fof$ is satisfied for $\str$. 
A sentence $\fof$ defines a language, the set of strings over $\alphabet$ satisfying it, denoted as $\lang(\fof)$.\looseness=-1
\begin{testexample}
The sentence below defines all the strings that contain $\syma\symb\symc$ as a subsequence (symbols appearing in order, possibly non-contiguously):
\begin{equation}
    \exists x:\left(\atom_\symc(x)\land\exists y<x: \left(\atom_\symb(y)\land \exists x<y: \atom_\syma(x)\right)\right).
\end{equation}
Note that the variable $x$ is reused, as permitted in $\pfo$.
The formulas below, however, are not definable in $\pfo$:
\begin{itemize}
    \item $x<y<z$ (more than two variables);
    \item $\forall x: (\exists y\geq x: \atom_\syma(y))$ (disallowed use of the quantifier $\exists y\geq x$).
\end{itemize}
\end{testexample}
$\pfo$ is specifically designed to match the expressive power of $\ptl$, allowing us to leverage both formalisms in our analysis of transformer expressivity.
\begin{restatable}{reTheorem}{pfotoptl}
\label{thm:pfo2ptl}
    $\pfo$ and $\ptl$ have the same expressive power.
\end{restatable}
\begin{proof}
    See \cref{app:pfo}.
\end{proof}

\section{Transformers}
In this section, we formally define the transformer idealization under consideration and establish two central expressivity results: (i) every transformer can be translated into an equivalent $\pfo$ formula, and (ii) every $\ptl$ formula can be simulated by a transformer. Combined with \cref{thm:pfo2ptl}, these results imply that transformers are expressively equivalent to both $\ptl$ and $\pfo$.

\subsection{The Transformer Architecture}
\label{sec:transformer}
The transformer considered in this section follows the architecture described by \citet{yang2024masked}, with the exception that we use \defn{soft attention} in place of unique hard attention (UHA). Before formally defining the architecture, we first state explicitly the assumptions that will govern our model as follows.
\begin{restatable}{reAssumption}{tassum}
\label{assumption:transformer}
We assume that the transformer satisfies the following conditions:
    \begin{enumerate}
        \item \defn{Fixed precision}: There exists a finite set of values that the floating-point numbers can assume. 
        We denote this set as $\fpnset=\{\fpn_1, \fpn_2, \ldots\}$.
        \item \defn{No positional encoding} (NoPE): Positional encodings are omitted for now, but are addressed in \cref{app:pe}.
        \item \defn{Soft attention}: Attention weights are computed using the $\softmax$ function. We also treat average hard attention (AHA) in \cref{app:hard}, showing that AHA is equally expressive as soft attention.
        \item \defn{Strict future masking}: As in \citet{yang2024masked}, the attention is strictly future-masked. Strict masking is shown to be more expressive than non-strict masking in \cref{app:non_strict}.
        \item \defn{Recognition}: Following standard practice in expressivity studies \citep{10.1162/tacl_a_00663}, the transformer is treated as a language recognizer. We will extend the analysis to transformer language models in \cref{sec:tlm}.
    \end{enumerate}
\end{restatable}

Transformers take strings as input. 
Given a string $\str=\sym_1\cdots\sym_\length$ over an alphabet $\alphabet$, we append a special end-of-sentence symbol $\eos \notin \alphabet$ to form a \defn{completed} string $\eosstr \defeq \sym_1 \cdots \sym_\length \eos$ over the extended alphabet $\eosalphabet \defeq \alphabet \cup \{\eos\}$. 
An initial embedding layer $\embedding \colon \eosalphabet^{\length+1} \to \fpnset^{\dimension \times (\length+1)}$ maps each input symbol including $\eos$ to a $\dimension$-dimensional representation.
We define the embedding layer as the $0^{\text{th}}$ layer of the transformer:
\begin{equation}
\transformer[0] \defeq \embedding.
\end{equation}

The embedding is followed by a stack of $\layernumber$ hidden layers. Each hidden layer contains two sublayers: a self-attention mechanism $\attention[\ell] \colon \fpnset^{\dimension \times (\length+1)} \rightarrow \fpnset^{\dimension \times (\length+1)}$, and a feedforward network $\ffn[\ell]\colon \fpnset^{\dimension \times (\length+1)} \rightarrow \fpnset^{\dimension \times (\length+1)}$.
These components are composed sequentially. 
For each $\ell \in \{0, 1, \ldots, \layernumber - 1\}$, we define:
\begin{subequations}
\begin{align}
    \transformer[\ell+0.5] &\defeq \LN\left(\attention[\ell]\left(\transformer[\ell]\right) + \transformer[\ell]\right), \\
    \transformer[\ell+1] &\defeq \LN\left(\ffn[\ell]\left(\transformer[\ell+0.5]\right) + \transformer[\ell+0.5]\right).
\end{align}
\end{subequations}
where $\LN$ denotes layer normalization \citep{ba2016layer}.
Finally, a classification head $\classification \colon \fpnset^\dimension \to \fpnset$ maps the representation at the $\eos$ position in the final layer to a scalar output:
\begin{equation}
    \cls(\eosstr) \defeq \classification\left(\transformer[\layernumber][:][\length+1][\eosstr]\right).
\end{equation}
We use $\transformer[\layernumber][d][n][\eosstr]$ to denote the $(d, n)^{\text{th}}$ entry in the matrix. Similarly, $\transformer[\layernumber][:][n][\eosstr]$ refers to the column vector at position $n$, and $\transformer[\layernumber][d][:][\eosstr]$ refers to the row vector corresponding to dimension $d$.
Consequently, a transformer defines a function of type $\alphabet^{\length+1}\rightarrow\fpnset$, where $\length$ is a parameter of the type.
We say a string $\str$ is accepted by the transformer if $\cls(\eosstr)>0$.
A detailed specification of the architecture is provided in \cref{app:transformer}.

\subsection{From Transformers to $\pfo$}
\label{sec:transformer2pfo}
In this subsection, we discuss the following theorem.
\begin{restatable}{reTheorem}{ttopfo}
\label{thm:transformer2pfo}
Every transformer can be simulated by $\pfo$.
\end{restatable}
\begin{proof}
    See \cref{app:transformer2pfo}.
\end{proof}
Our proof closely follows the approach of \citet[Section 5]{pmlr-v202-chiang23a}, with one key difference: the simulation of summation in soft attention.
While prior work makes use of counting quantifiers to simulate summation, \citet{li2024chain} demonstrate that summation with iterative rounding can be simulated with $\fo$. We take this a step further by showing that the specific summation involved in soft attention can be simulated by $\pfo$. 
Here, we provide a high-level overview of the proof. We identify two summations in soft attention. The first occurs in the denominator of the $\softmax$ function, defined for a $\dimension$-dimensional vector $\vx$, as follows:
\begin{equation}
    \softmax(\vx)_d \defeq \frac{\exp(x_d)}{\sum_{i=1}^D \exp(x_i)}, \text{ for } d\in\{1,\ldots,\dimension\}.
\end{equation}
$\exp$ is a deterministic function. Under the fixed precision assumption, the possible input and output values of $\exp$ form a finite set, allowing $\pfo$ formulas to encode $\exp$ by explicit enumeration of these values.
Since $\exp$ outputs only non-negative values, the summation in the denominator reduces to a sum of non-negative terms, which can be simulated by $\pfo$ (\cref{lemma:sum_pos}).
Additionally, the second summation is a weighted sum, in which the weights are produced by the $\softmax$ function. Under fixed precision, the output of a $\softmax$ contains a bounded number of non-zero entries (\cref{prop:softmax_finite}). Therefore, the weighted sum involves only a bounded number of terms and can similarly be simulated by $\pfo$ (\cref{lemma:sum_bounded}).\looseness=-1

\subsection{From $\ptl$ to Transformers}
\label{sec:ptl2transformer}
In this subsection, we discuss the following theorem.
\begin{restatable}{reTheorem}{ptltot}
\label{thm:ptl2transformer}
Every $\ptl$ formula can be simulated by a transformer.
\end{restatable}
\begin{proof}
    See \cref{app:ptl2transformer}.
\end{proof}
Our proof is adapted from that of \citet[Appendix C.1]{yang2024masked}. Intuitively, each $\ptl$ formula is encoded in a dedicated coordinate $d \in \{1,\ldots,\dimension\}$ of the transformer's hidden state. At each position, this coordinate takes the value $1$ if the formula is satisfied at that position, and $0$ otherwise.
The key challenge in proving \cref{thm:ptl2transformer} lies in simulating the temporal operator $\past$.
Prior constructions typically employ uniform attention, which assigns equal weight to all preceding positions, to simulate such operators. 
However, under fixed precision, soft attention has a limited attention span, as the output of the $\softmax$ function contains only a bounded number of non-zero entries. 
We refer to this bound as the maximum attention span, denoted by $\nmax$. 
To overcome this limitation, previous work has relied on non-fixed numerical precision---such as arbitrary precision \citep{pmlr-v202-chiang23a, yang2024countingliketransformerscompiling} or log precision \citep{NEURIPS2023_a48e5877}---to prevent attention weights from vanishing. In contrast, we present a construction that overcomes this issue without requiring increased numerical precision.\looseness=-1

We begin by describing a base construction that applies when attention weights do not vanish, and then extend it to handle cases where vanishing weights may occur.
\paragraph{Base Construction.} Suppose a formula $\tlf$ is simulated by coordinate $d_1$ at the $\ell^{\text{th}}$ layer, i.e., by $\transformer[\ell][d_1][:]$ (omitting the input string $\eosstr$ for brevity), and let $\past \tlf$ be the formula we aim to simulate. We construct an attention sublayer that uniformly attends to all previous positions $m < n$ such that $\transformer[\ell][d_1][m] = 1$. 
If such positions exist, we set an unused coordinate $\coord{\past}$ to 1, i.e., $\transformer[\ell+0.5][\coord{\past}][n] = 1$ and otherwise to $0$.\looseness=-1

\paragraph{Handling Vanishing Weights.}
However, when too many (more than $\nmax$) previous positions satisfy $\tlf$, the resulting attention weights will underflow to $0$. This results in $\transformer[\ell+0.5][\coord{\past}][n] = 0$ even when $\past \tlf$ is true, leading to incorrect behavior. 
To address this, we first compute the logical conjunction of $d_1$ and $\coord{\past}$ and store it in $\coord{\land}$, i.e., 
\begin{equation}
    \transformer[\ell+1][\coord{\land}][n] = 
    \begin{cases}
        1 & \text{if } \transformer[\ell][d_1][n]=1 \text{ and } \transformer[\ell+0.5][\coord{\past}][n] = 1, \\
        0 & \text{otherwise}.
    \end{cases}
\end{equation}
The number of positions $n$ with $\transformer[\ell+1][\coord{\land}][n] = 1$ is bounded by $\nmax$, since beyond the $(\nmax+1)^{\text{th}}$ position, the first attention sublayer will suffer from vanishing attention weights, causing $\transformer[\ell+0.5][\coord{\past}][n] = 0$.\looseness=-1

Next, we construct a second layer in which the attention sublayer uniformly attends to all positions $m < n$ where $\transformer[\ell+1][\coord{\land}][m] = 1$. This produces a coordinate $\coord{\past\past}$ that simulates $\past(\tlf\land\past\tlf)$. Since the number of positions satisfying $\transformer[\ell+1][\coord{\land}][m] = 1$ is bounded, this second attention sublayer is not subject to vanishing attention weights and therefore computes its output reliably.

The formula $\past(\tlf\land\past\tlf)$ effectively asks whether there are at least two positions before $n$ that satisfy $\tlf$. Thus, if there is exactly one position $m < n$ satisfying $\tlf$, then $\past(\tlf\land\past\tlf)$ differs from the target formula $\past\tlf$. In this case, however, the first attention sublayer (producing $\coord{\past}$) correctly simulates $\past\tlf$,  since the attention only needs to cover one position.
We therefore take the logical disjunction of $\coord{\past}$ and $\coord{\past\past}$, implemented via the subsequent feedforward sublayer. The resulting coordinate $\coord{\lor}$ correctly simulates the formula $\past\tlf$. This completes the proof sketch.

\begin{testexample}
    Assume the attention span is limited to 1 position. The following example illustrates the simulation process, with defective simulations highlighted in \textcolor{ETHRed}{red}.
    \begin{center}
        \begin{tabular}{rcccccc}
            formula & coordinate & \multicolumn{5}{c}{position} \\
            $\tlf$ & $d_1$ & $1$ & $0$ & $1$ & $0$ & $1$ \\
            $\textcolor{ETHRed}{\past\tlf}$ & $\coord{\past}$ & 0 & 1 & 1 & \textcolor{ETHRed}{0} & \textcolor{ETHRed}{0} \\
            $\tlf\land\textcolor{ETHRed}{\past\tlf}$ & $\coord{\land}$ & 0 & 0 & 1 & 0 & \textcolor{ETHRed}{0} \\
            $\past(\tlf\land\past\tlf)$ & $\coord{\past\past}$ & 0 & 0 & 0 & 1 & 1 \\
            $\past\tlf=\past(\tlf\land\past\tlf)\lor\textcolor{ETHRed}{\past\tlf}$ & $\coord{\lor}$ & 0 & 1 & 1 & 1 & 1 \\
        \end{tabular}	
    \end{center}
\end{testexample}

\section{Characterizations of $\ptl$}
\label{sec:characterizations}
We have established the expressive equivalence between transformers and $\ptl$. This connection becomes especially compelling when paired with precise and rich characterizations of $\ptl$. To that end, we prove the following theorem.
\begin{restatable}{reTheorem}{thmptl}
    \label{thm:ptl}
    Let $\lang\subseteq\kleene{\alphabet}$ be a regular language, $\monoid$ be its syntactic monoid, and $\automaton$ be the DFA accepting it. The following assertions are equivalent: (1) $\lang$ is a left-deterministic polynomial,  (2) $\monoid$ is $\gR$-trivial, (3) $\automaton$ is partially ordered, and (4) $\lang$ is definable by $\ptl$.
\end{restatable}
\begin{proof}
    See \cref{app:characterizations}.
\end{proof}

In the remainder of this section, we focus on the PODFA characterization, which will be central to later developments. The other characterizations are defined in \cref{app:characterizations}, which also includes a discussion of superclasses of $\ptl$, summarized in \cref{tab:equivalence}.

\begin{definition}
    A \defn{deterministic finite automaton} (DFA) is a $5$-tuple $\automaton=\dfatuple$ where
    \begin{itemize}
    \item $\alphabet$ is an alphabet;
    \item $\states$ is a finite set of states;
    \item $\qinit\in\states$ is the initial state;
    \item $\final\subseteq \states$ is the set of final states;
    \item $\trans\colon\states\times\alphabet\rightarrow\states$ is a total transition function;
    \item $\reject\in\states\setminus\final$ is a rejecting sink state, i.e., $\trans(\reject,\syma)=\reject$ for every $\syma\in\alphabet$.
    \end{itemize}
\end{definition}

Given an automaton $\automaton=\dfatuple$, we say $\automaton$ is in state $\stateq$ upon reading a string $\str=\sym_1\cdots\sym_\length\in\kleene{\alphabet}$ if and only if there exists a sequence of states $\stateq_0,\ldots,\stateq_\length$ such that    
\begin{itemize}
    \item $\stateq_0=\qinit$;
    \item $\stateq_{n+1}=\trans(\stateq_{n},\sym_{n+1})$ for $n \in [\length-1]$;
    \item $\stateq_\length=\stateq$.
\end{itemize}
A string $\str$ is accepted by $\automaton$ if it reaches one of the final states upon reading $\str$.

\begin{definition}
    A DFA is said to be \defn{partially ordered} if there exists a partial order $\preceq$ on $\states$ such that for every state $\stateq \in \states$ and symbol $\syma \in \alphabet$, we have $\stateq \preceq \trans(\stateq, \syma)$.
\end{definition}
Informally, in a partially ordered DFA, once the automaton leaves a state, it never revisits it. An example of a non-partially ordered DFA is provided in \cref{app:po}.

\section{Transformer Language Models}
\label{sec:tlm}
In this section, we extend our analysis to transformer LMs. Formally, an LM constitutes a distribution over $\kleene{\alphabet}$, typically factorized autoregressively as follows:
\begin{equation}
    \plm(\str) \defeq \localplm(\eos\mid\str)\prod_{n=1}^{\length} \localplm(\sym_n \mid \str_{<n}).
\end{equation}
where $\str_{<n}\defeq\sym_1\cdots\sym_{n-1}$. We define an LM $\plm$'s prefix probability $\preprob$ as:
\begin{equation}
    \preprob(\str) \defeq \sum_{\str'\in\kleene{\alphabet}} \plm(\str\str').
\end{equation}
To adapt a transformer for language modeling, we feed the input string $\str$ into the model autoregressively, and the classification head is replaced with a language modeling head $\lmhead \colon \fpnset^\dimension \rightarrow \fpnset^{\eosalphabetsize}$, which maps the representation in the last layer $\layernumber$ at the most recent position $n-1$ to a probability distribution over $\eosalphabet$. This defines the local distribution $\localplm$ as follows:
\begin{equation}
    \localplm\left(\sym_n\mid \str_{<n}\right) \defeq \lmhead\left(\transformer[\layernumber][:][n-1][\str_{<n}]\right)_{\sym_n}.
\end{equation}
where $\lmhead(\transformer[\layernumber][:][n-1][\str_{<n}])_{\sym_n}$ denotes the probability assigned to the symbol $\sym_n$. 
Under fixed precision, this distribution may not be perfectly normalized, as the sum of its components can deviate from $1$ due to rounding errors.

As transformer LMs do not introduce any new computational components beyond those already present in transformer recognizers, the same constructions and proofs apply. Therefore, every transformer LM can be simulated by $\pfo$.
\begin{restatable}{reTheorem}{tlmtopfo}
\label{thm:tlm2pfo}
Every transformer LM can be simulated by $\pfo$.
\end{restatable}
\begin{proof}
    See \cref{app:tlm}.
\end{proof}

The converse direction is more subtle. \citet{yang2024countingliketransformerscompiling} show how to compile a transformer LM from C-RASP. Here, we present an alternative approach that is more intuitive. Our construction leverages the PODFA characterization of $\ptl$. 
Given a PODFA $\automaton = \dfatuple$, we can define a corresponding LM $\blm$ by specifying the local distribution $\localblm$ as follows: Suppose $\automaton$ is in state $\stateq$ upon reading the prefix $\str_{<n}$:
\begin{itemize}
    \item If $\stateq\notin\final$ and $\stateq\neq\reject$, $\localblm(\syma \mid \str_{<n})$ assigns uniform probability over all symbols in $\{\syma\mid\trans(\stateq, \syma) \in \states\setminus\{\reject\}\}$;\looseness=-1
    \item If $\stateq\in\final$, $\localblm(\syma \mid \str_{<n})$ assigns uniform probability over all symbols in $\{\syma\mid\trans(\stateq, \syma) \in \states\setminus\{\reject\}\}\cup\{\eos\}$;
    \item If $\stateq = \reject$, $\localblm(\syma \mid \str_{<n}) = 0$ for all $\syma \in \eosalphabet$, i.e., $\preprob(\str_{<n}) = 0$. 
\end{itemize}
Note that in practice, a $\softmax$ cannot produce an all-zero vector. For technical convenience, we therefore extend the $\eosalphabet$ with a special symbol $\unk\notin\eosalphabet$, which receives all the probability mass when $\automaton$ reaches the rejecting state, i.e., $\localblm(\unk \mid \str_{<n}) = 1$ if $\stateq = \reject$. 
Accordingly, we extend the LM head $\lmhead$ to accommodate this special symbol: $\lmhead \colon \fpnset^\dimension \rightarrow \fpnset^{|\eosalphabet\cup\{\unk\}|}$.

\begin{restatable}{reTheorem}{pototlm}
\label{thm:po2tlm}
    Every PODFA LM can be simulated by a transformer LM.
\end{restatable}
\begin{proof}
    See \cref{app:tlm}.
\end{proof}

Intuitively, for each state $\stateq \in \states$, we can construct a $\ptl$ formula $\tlf_\stateq$ such that $\str_{<n}, n-1 \models \tlf_\stateq$ if and only if the automaton is in state $\stateq$ after reading the prefix $\str_{<n}$. By \cref{thm:ptl2transformer}, each formula $\tlf_\stateq$ can be encoded into a designated coordinate $d_\stateq$ of the transformer's hidden state. The language modeling head, extended to include the special token $\unk$, then maps these coordinates to the corresponding probability distribution specified by $\localblm$.\looseness=-1 

\begin{table}[b]
    \centering\small
    \caption{Language recognition experiments. Language classes are ordered by decreasing complexity. For each class, examples are chosen to be minimally more complex than those in the immediately lower class. Maximum and mean accuracies (± standard deviation) are reported. Exact values of $100.0\%$ accuracy are highlighted in bold.}
    \label{tab:results}
    \begin{tabular}{ccccccc} 
    \toprule
        \multirow{2}{*}{Class} & \multirow{2}{*}{Language} & \multicolumn{2}{c}{Transformer} & \multicolumn{2}{c}{LSTM} \\
        & & max ($\%$) & mean ($\%$) & max ($\%$) & mean ($\%$) \\
        \midrule
        Counter languages & \textsc{cnt} & $83.3$ & $53.6\pm 8.6$ & $\mathbf{100.0}$ & $86.9\pm20.9$ \\
        \midrule
        Regular languages & \textsc{parity} & $52.1$ & $50.6\pm 0.8$ & $\mathbf{100.0}$ & $\mathbf{100.0}\pm 0.0$ \\
        \midrule
        \multirow{3}{*}{Star-free} & \textsc{dyck}-$(1,2)$ & $83.4$ & $64.2\pm8.0$ & $\mathbf{100.0}$ & $99.3\pm1.0$ \\
        & \textsc{dyck}-$(1,1)$ & $87.7$ & $71.5\pm9.3$ & $\mathbf{100.0}$ & $88.8\pm17.3$ \\
        & \textsc{lt}-2 & $62.1$ & $57.9\pm2.3$ & $\mathbf{100.0}$ & $\mathbf{100.0}\pm0.0$ \\
        \midrule
        \multirow{2}{*}{Unambiguous polynomials} & \textsc{rdp}-1 & $90.0$ & $71.1\pm11.5$ & $\mathbf{100.0}$ & $\mathbf{100.0}\pm0.0$ \\
        & \textsc{last} &$64.8$ & $57.3\pm2.7$ & $\mathbf{100.0}$ & $\mathbf{100.0}\pm0.0$ \\
        \midrule
        \multirow{4}{*}{Left-deterministic polynomials} & \textsc{pt}-2 & $\mathbf{100.0}$ &$98.3\pm3.5$ & $\mathbf{100.0}$ &$99.7\pm1.1$ \\
        & \textsc{lt}-1 &$\mathbf{100.0}$ & $88.8\pm11.8$ &$\mathbf{100.0}$ & $93.0\pm14.2$ \\
        & \textsc{ldp}-2 & $\mathbf{100.0}$ & $100.0\pm0.0$ &$\mathbf{100.0}$ & $\mathbf{100.0}\pm0.0$ \\
        & \textsc{ldp}-1 & $\mathbf{100.0}$ &$97.3\pm6.0$ & $\mathbf{100.0}$ & $\mathbf{100.0}\pm0.0$ \\
        & \textsc{first} & $\mathbf{100.0}$ &$99.4\pm1.4$ & $\mathbf{100.0}$ & $\mathbf{100.0}\pm0.0$ \\
        \toprule
    \end{tabular}
\end{table}

\section{Experiments}
\label{sec:experiments}
The experiments are divided into two parts: language recognition and language modeling. 
\subsection{Language recognition}
\label{sec:lr}
We have shown that a transformer can recognize only left-deterministic polynomials. In this section, we conduct a series of language classification experiments to empirically validate this claim.
We consider five language classes, arranged in a strict inclusion hierarchy---each class is a proper subset of the one preceding it. For each class, we select one or more representative languages. These languages are listed in \cref{tab:results}, and their detailed definitions are provided in \cref{app:tasks}.

\subsubsection{Experimental setup}
We implement the transformer that we theorize about (\cref{assumption:transformer}). For comparison, we also train a long short-term memory (LSTM) \citep{Hochreiter97}. Models are trained on strings up to length $40$, and tested on strings of length $41--500$.
Each experiment is run with $5$ different random seeds and $3$ learning rates. 
We consider a transformer to have successfully recognized a language if it achieves $100\%$ accuracy in at least one of the runs. Details of the experimental setup and model configurations are provided in \cref{app:experimental_setup}.

\subsubsection{Results}
We compute classification accuracy and report both the maximum and mean values across all runs in \cref{tab:results}. The LSTM achieves perfect accuracy on all tasks, consistent with previous work showing that LSTMs can recognize regular languages \citep{merrill-2019-sequential} and implement counting mechanisms \citep{weiss-etal-2018-practical}. This confirms that the tasks are learnable given the available training data. 
Results on transformers align precisely with our theoretical predictions: under fixed precision, transformers with soft attention and NoPE can recognize exactly the class of left-deterministic polynomials. They achieve perfect accuracy on all $\ptl$-definable languages but consistently fail on tasks outside this class, even though prior work has shown that some of these languages are theoretically recognizable under more permissive idealizations \citep{pmlr-v202-chiang23a, NEURIPS2023_a48e5877, yang2024countingliketransformerscompiling, yao-etal-2021-self, svete-cotterell-2024-transformers}.  Notably, we use single-precision (32-bit) floating-point numbers, and the string lengths never exceed the maximum attention span of the transformer. That is, attention can uniformly cover all prior positions without numerical underflow or overflow. Yet, despite these favorable conditions, the transformer exhibits no expressive power beyond what is predicted by our formal characterization. A more detailed breakdown of the results can be found in \cref{app:results}.\looseness=-1

\subsection{Language modeling}
We now turn to experiments on language modeling, focusing on three representative languages: \textsc{rdp}-1, \textsc{ldp}-2, and \textsc{ldp}-1. The corresponding automata are illustrated in \cref{fig:automata}.

We train the transformer LMs using the standard cross-entropy loss. For evaluation, a predicted symbol is considered correct if it has non-zero probability under the target distribution $\localblm$ induced by the DFA.
Per-token accuracies are reported in \cref{tab:lm}. The transformer LM successfully learns \textsc{ldp}-1 and \textsc{ldp}-2 with perfect accuracy. For \textsc{rdp}-1, the best performance reaches $98.3\%$, but the model consistently falls short of achieving $100.0\%$. This gap becomes more evident upon inspecting the hidden states of the model. 

\begin{figure*}[b]
    \centering\small
    \begin{subfigure}[t]{0.28\textwidth}
        \centering
        \includegraphics[scale=1]{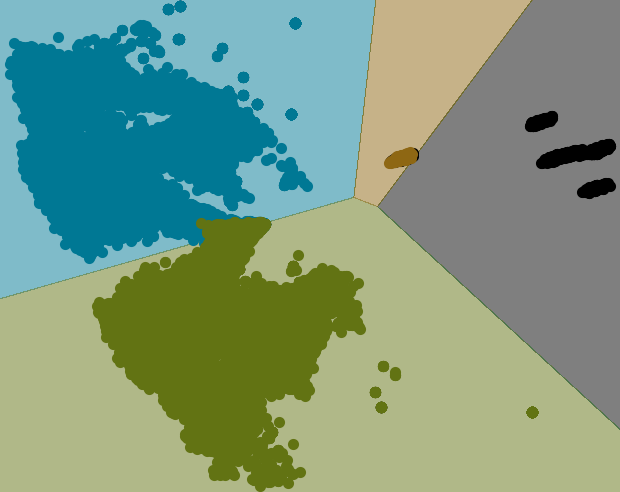}
        \caption{\textsc{rdp}-1}
    \end{subfigure}
    \begin{subfigure}[t]{0.28\textwidth}
        \centering
        \includegraphics[scale=1]{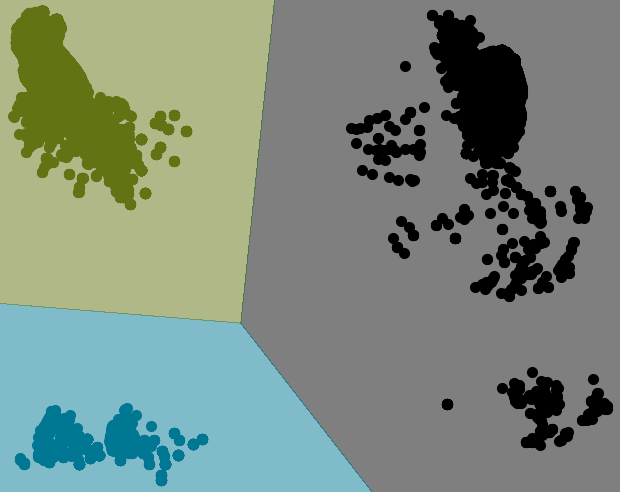}
        \caption{\textsc{ldp}-1}
    \end{subfigure}
    \begin{subfigure}[t]{0.4\textwidth}
        \centering
        \includegraphics[scale=1]{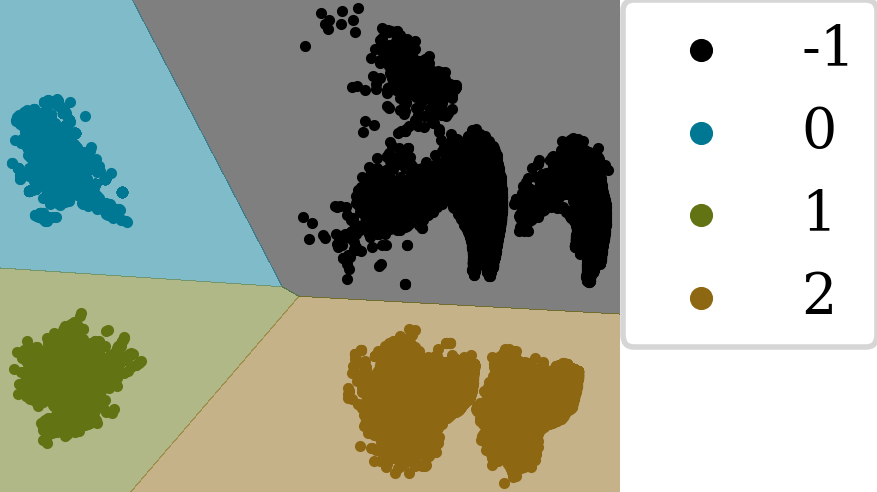}
        \captionsetup{singlelinecheck=false,margin={1.4cm,0cm}}
        \caption{\textsc{ldp}-2}
    \end{subfigure}
    \caption{Visualization of the transformer's representations for different DFA states. $-1$ denotes the rejecting state $\reject$. The filled contours represent the decision boundaries of a linear classifier.}
    \label{fig:voronoi}
    \vspace{-10pt}
\end{figure*}

We extract the representation $\transformer$ at each position $n$ by concatenating the outputs of all self-attention and feedforward sublayers:
\begin{equation}
    \transformer_{:,n} = \left[
            \transformer[0.5][:][n][\eosstr]^\top
            \cdots 
            \transformer[\layernumber][:][n][\eosstr]^\top
        \right]^\top.
\end{equation}

To assess whether the model internally tracks the states in the DFAs, we train a linear classifier to predict the state the DFA is in upon reading $\str_{<n}$, given $\transformer_{:,n-1}$. The classifier consists of two linear layers with no intermediate nonlinearity: the first projects to $\fpnset^2$ for visualization, and the second performs classification. In \cref{fig:voronoi}, we plot the 2D projections and overlay the decision boundaries of the classifier.\looseness=-1

The visualizations reveal that, for \textsc{ldp}-1 and \textsc{ldp}-2, the transformer does distinguish the states, as their representations are linearly separable. In contrast, for \textsc{rdp}-1, state representations are intermixed. Notably, even when we train a neural probe with nonlinearity on the representations, perfect probing accuracy remains elusive. This suggests that the transformer does not learn to fully separate the states in the non-partially ordered DFA corresponding to \textsc{rdp}-1.

\section*{Acknowledgments}
This publication was made possible by an ETH AI Center doctoral fellowship to Jiaoda Li.

\bibliography{reference}
\bibliographystyle{plainnat}

\appendix

\section{Background}
In this section, we provide background information that was omitted from the main text for brevity. We include it here for the sake of completeness.

\subsection{Temporal logic}
Temporal logic \citep{pnueli1977} is a special case of modal logic, with modalities referring to time. 
\label{app:tl}
\subsubsection{$\protect\ltl$}
Linear temporal logic $\ltl$ includes four temporal operators: $\past$ (\texttt{past}), $\future$ (\texttt{future}), $\since$ (\texttt{since}), and $\until$ (\texttt{until}). Notably, $\past$, $\future$, and one of $\since$ or $\until$ can be omitted without loss of expressive power \citep{10.1145/567446.567462}. 
We write $\true$ to denote $\TRUE$ and $\false$ to denote $\FALSE$. The semantics of the formulas are defined inductively as follows:
\begin{itemize}
    \item $\str,n\models \atom_\syma$ iff the $n^{\text{th}}$ symbol of $\str$ is an $\syma$. 
    \item $\str,n\models \tlf_1 \land \tlf_2$ iff $\str,n\models \tlf_1$ and $\str,n\models \tlf_2$. 
    \item $\str,n\models \lnot \tlf$ iff $\str,n \not\models \tlf$.
    \item $\str,n\models \past \tlf$ iff there exists $m$ such that $m<n$ and $\str, m\models \tlf$. 
    \item $\str,n\models \future \tlf$ iff there exists $m$ such that $m>n$ and $\str, m\models \tlf$.
    \item $\str,n\models \tlf_1 \since \tlf_2$ iff there exists $m$ such that $m<n$, $\str,m\models\tlf_2$, and for every $k$ such that $m<k<n$, $\str, k\models \tlf_1$.
    \item $\str,n\models \tlf_1 \until \tlf_2$ iff there exists $m$ such that $m>n$, $\str,m\models\tlf_2$, and for every $k$ such that $n<k<m$, $\str, k\models \tlf_1$.
\end{itemize}

\begin{testexample}
    \begin{itemize}[leftmargin=*]
        \item $\syma\symb\symb,1 \models \atom_\syma$ as the first symbol is $\syma$.
        \item $\syma\symb\symb,3 \models \past\atom_\syma$ as an $\syma$ occurs before position 3.
        \item $\syma\symb\symb,3 \models \atom_\symb\since \atom_\syma$ as $\atom_\syma$ holds at position 1 (before position 3), and $\atom_\symb$ holds at every position between 1 and 3.
    \end{itemize}
\end{testexample}

To say a string $\str$ of length $\length$ satisfies a formula $\tlf$, written as $\str\models\tlf$, we imagine we start at a position outside the string, e.g., position $0$ or $\length+1$.

\begin{testexample}
    \begin{itemize}[leftmargin=*]
        \item $\syma\symb\symb\models\atom_\symb\since\atom_\syma$ as the formula is satisfied at position $4$, outside the string.
        \item $\syma\symb\symb\models\atom_\syma\until\atom_\symb$ as the formula is satisfied at position $0$, outside the string.
    \end{itemize}
\end{testexample}

The language defined by a formula $\tlf$ is the set of all strings that satisfy $\tlf$, denoted by $\lang(\tlf)$:
\begin{equation}
    \lang(\tlf) \defeq \{\str\in\kleene{\alphabet} \mid \str\models\tlf\}.
\end{equation}

\begin{testexample}
    \begin{subequations}
    \begin{align}
        \lang\left(\atom_\symb\since\atom_\syma\right) &= \kleene{\alphabet}\syma\kleene{\symb} \\
        \lang\left(\atom_\syma\until\atom_\symb\right) &= \kleene{\syma}\symb\kleene{\alphabet}
    \end{align}
    \end{subequations}
\end{testexample}

We define the \defn{operator depth} of a formula $\tlf$, denoted $\od(\tlf)$, as the maximum number of nested temporal operators in $\tlf$. It is defined recursively as follows:
\begin{equation*}
    \begin{aligned}
        \od(\atom_\syma) &= 0 \\
        \od(\tlf_1\land\tlf_2) &= \max(\od(\tlf_1),\od(\tlf_2)) \\
        \od(\lnot\tlf) &= \od(\tlf) \\
        \od(\past\tlf) &= \od(\tlf) + 1 \\
        \od(\future\tlf) &= \od(\tlf) + 1 \\
        \od(\tlf_1\since\tlf_2) &= \max(\od(\tlf_1),\od(\tlf_2)) + 1 \\
        \od(\tlf_1\until\tlf_2) &= \max(\od(\tlf_1),\od(\tlf_2)) + 1
    \end{aligned}
\end{equation*}

\subsubsection{$\protect\tl$}
Unary temporal logic $\tl$ is the fragment of $\ltl$ that excludes the binary operators $\since$ and $\until$. This restriction renders $\tl$ strictly less expressive than $\ltl$ \citep{561310, ETESSAMI2002279}.
\begin{testexample}
    Consider the language consisting of all words that begin with $\syma$ and end with $\symb$. To enforce that a word starts with $\syma$, we require that $\syma$ occurs at the beginning of the word, i.e., it is preceded by nothing. This can be expressed as $\past(\atom_\syma \land \lnot \past\true)$. Symmetrically, ending with $\symb$ can be expressed using the dual condition involving $\future$. Thus, the formula
    \begin{equation}
    \lang(\past(\atom_\syma \land \lnot \past\true) \land \past(\atom_\symb \land \lnot \future\true)) = \syma\kleene{\alphabet}\symb
    \end{equation}
    defines the desired language.
\end{testexample}

\subsubsection{$\ptl$}
We define $\ptl$ as the past fragment of $\tl$, which excludes the $\future$ operator. 
\begin{proposition}
    $\ptl$ is strictly less expressive than $\tl$. 
\end{proposition}
\begin{proof}
    Since $\ptl$ is a syntactic fragment of $\tl$, every formula expressible in $\ptl$ is also expressible in $\tl$. To establish strictness, it suffices to exhibit a language definable in $\tl$ that is not definable in $\ptl$. Consider again the language $\syma\kleene{\alphabet}\symb$. It is not definable in $\ptl$, as expressing the end of a string requires the ability to look forward, which $\ptl$ lacks.
\end{proof}

\subsection{First-order logic}
The seminal work of \citet{Büchi1990} demonstrated how properties of languages can be described using logical formulas, while \citet{mcnaughton1971counter} was the first to restrict these formulas to first-order logic ($\fo$).
\label{app:fo}
\subsubsection{$\fo$}
The atomic formulas of $\fo$ consist of unary predicates $\atom_\syma$ for each $\syma\in\alphabet$ and a binary numerical predicate $<$ that can be used to determine the sequential order between variables. 
Formulas in $\fo$ may use arbitrarily many variables. However, it has been shown that three variables are both necessary and sufficient to define all first-order definable languages \citep{Kamp1968-KAMTLA-2, IMMERMAN1989121}.
In addition to the usual Boolean connectives $\land$ and $\lnot$, there is an \defn{existential quantifier} $\exists$.
The formulas in $\fo$ are constructed inductively as follows:
\begin{itemize}
    \item Every atomic formula is a formula;
    \item A Boolean combination of formulas is a formula;
    \item If $\fof(x,\ldots)$ is a formula, so is $\exists x: \fof(x,\ldots)$;
\end{itemize}
The \defn{universal quantifier} $\forall$ can be defined as follows:
\begin{equation}
    \forall x: \fof(x,\ldots) \defeq \lnot \exists x: \lnot \fof(x,\ldots).
\end{equation}
A formula with no free variables is called a \defn{sentence}. We write $\str\models\fof$ if $\fof$ is satisfied for $\str$ under the interpretation, where variables range over the positions in $\str$, unary predicates $\atom_\syma(x)$ hold if the symbol at position $x$ is $\syma$, and $<$ is interpreted as the natural order on positions.
A sentence $\fof$ can define a language, the set of strings over $\alphabet$ satisfying it, denoted as $\lang(\fof)$.\looseness=-1 

\begin{testexample}
  The formula below defines the language $\syma\kleene{\alphabet}\symb$:
  \begin{equation}
    \label{eg:fo}
    \exists x \forall y \exists z: (\atom_\syma(x)\land x\leq y\land y\leq z\land \atom_\symb(z))
  \end{equation}
  where $x \leq y$ is shorthand for $x < y \lor x = y$. 
\end{testexample}

The \defn{quantifier depth} of a formula $\fof$, denoted $\qd(\fof)$, is defined recursively as follows:
\begin{equation*}
    \begin{aligned}
        \qd(\atom_\syma) &= 0 \\
        \qd(\fof_1\land\fof_2) &= \max(\qd(\fof_1),\qd(\fof_2)) \\
        \qd(\lnot\fof) &= \qd(\fof) \\
        \qd(\exists x:\fof(x,\ldots)) &= \qd(\fof) + 1 \\
    \end{aligned}
\end{equation*}

It is well known that $\fo$ and $\ltl$ are equivalent in expressive power \citep{Kamp1968-KAMTLA-2} and they both define exactly the class of star-free languages \citep{mcnaughton1971counter}.

\subsubsection{$\fo[2]$}
$\fo[2]$ is the fragment of $\fo$ restricted to using only two distinct variables (typically reused via quantification).
\begin{testexample}
  The formula \cref{eg:fo} can be written as a formula in $\fo[2]$ as follows: 
  \begin{equation}
    \exists x: \left(\atom_\syma(x)\land \forall y:\left(x\leq y\land \exists x:(y\leq x\land \atom_\symb(x))\right)\right),
  \end{equation}
  which uses only two distinct variables $x$ and $y$. 
\end{testexample}

It is shown that $\tl$ and $\fo[2]$ recognize precisely the same languages \citep{ETESSAMI2002279}.

\subsubsection{$\pfo$}
\label{app:pfo}
We define $\pfo$ as the past fragment of $\fo[2]$, in which formulas are restricted so that whenever a free variable $x$ is present, every existential quantifier must be of the form $\exists y < x$.. This ensures that all quantification is limited to positions at or before $x$, i.e., only past positions relative to the free variable are accessed.

$\pfo$ has precisely the same expressive power as $\ptl$. We establish this equivalence by providing translations in both directions.
\begin{lemma}
  \label{lemma:ptl2pfo}
  Every formula $\tlf$ in $\ptl$ can be translated into an equivalent single-variable formula $\fof(x)$ in $\pfo$.
\end{lemma}
\begin{proof}
    We proceed by structural induction on the formula $\tlf$.
    \paragraph{Base case.} If $\tlf = \atom_\syma$, we translate it to $\fof(x) = \atom_\syma(x)$.
    \paragraph{Induction step.} Assume $\tlf_1$ and $\tlf_2$ can be translated into $\fof_1(x)$ and $\fof_2(x)$ respectively. We translate compound formulas as follows:
    \begin{itemize}
    \item If $\tlf=\tlf_1\land\tlf_2$, then $\fof(x) = \fof_1(x)\land\fof_2(x)$.
    \item If $\tlf=\lnot \tlf_1$, then $\fof(x) = \lnot \fof_1(x)$.
    \item If $\tlf=\past\tlf_1$, then $\fof(x) = \exists y<x: \fof_1(y)$
    \end{itemize}
    In each case, the resulting formula $\fof(x)$ belongs to $\pfo$, since it uses at most two variables (the free variable $x$ and a bound variable $y$), both of which can be reused.
\end{proof}

\begin{lemma}
  \label{lemma:pfo2ptl}
  Every single-variable formula $\fof(x)$ in $\pfo$ can be translated into an equivalent formula $\tlf$ in $\ptl$.
\end{lemma}
\begin{proof}
  We proceed by induction on the quantifier depth $\qd(\fof)$.
  \paragraph{Base case.} If $\qd(\fof) = 0$, then $\fof(x)$ is a Boolean combination of atomic formulas such as $\atom_\syma(x)$, which are directly translatable into $\ptl$ formulas.
  \paragraph{Induction step.} 
    Assume that every formula of quantifier depth at most $k$ can be translated into an equivalent $\ptl$ formula. Consider a formula $\fof(x)$ of depth $k + 1$. The formula can be written as a Boolean combination of:
    \begin{itemize}
        \item Subformulas of quantifier depth at most $k$, which can be translated into $\ptl$ by the inductive hypothesis;
        \item Subformulas of the form $\exists y < x: \fof_1(y)$, where $\fof_1$ has depth at most $k$. By the inductive hypothesis, $\fof_1(y)$ translates to a $\ptl$ formula $\tlf_1$, and the whole subformula translates to $\tlf = \past \tlf_1$.
    \end{itemize}
\end{proof}

Combining the two lemmas, we obtain the following theorem.
\pfotoptl*
\begin{proof}
    The result follows directly from \cref{lemma:ptl2pfo} and \cref{lemma:pfo2ptl}. To define languages, $\ptl$ formulas are interpreted at position $\length + 1$. Similarly, $\pfo$ formulas with a single variable $x$ can be converted into sentences by prefixing them with $\exists x$, which is semantically equivalent to $\exists x < \length + 1$.
\end{proof}

\section{Transformers}
In this section, we describe the transformer idealization under consideration in more detail. The formal proofs for the following two results are also given: (i) every transformer can be translated into an equivalent $\pfo$ formula, and (ii) every $\ptl$ formula can be simulated by a transformer.

\subsection{Transformer architecture}
\label{app:transformer}
The assumptions we make are restated as follows:
\tassum*
We denote the set of non-negative floating-point numbers as $\nnfpnset \subset \fpnset$. Similar notation is used for positive numbers, negative numbers, and other subsets. The set $\fpnset$ includes two special values, $\infty$ and $-\infty$, representing positive and negative infinity. Their behaviors are defined as follows:
\begin{multicols}{2}
\begin{itemize}
  \item $\forall\fpn\in\fpnset\setminus\{-\infty\}$, $\infty + \fpn \defeq \infty$;
  \item $\forall\fpn\in\fpnset\setminus\{-\infty\}$, $-\infty + \fpn \defeq -\infty$;
  \item $\forall\fpn\in\fpnset_{>0}$, $\fpn\cdot\infty\defeq\infty$ and $\fpn\cdot(-\infty)\defeq-\infty$;
  \item $\forall\fpn\in\fpnset_{<0}$, $\fpn\cdot\infty\defeq-\infty$ and $\fpn\cdot(-\infty)\defeq\infty$;
  \item $\forall\fpn\in\fpnset\setminus\{\infty, -\infty\}$, $\fpn/\infty\defeq0$;
  \item $\exp(\infty)\defeq\infty$ and $\exp(-\infty)\defeq 0$.
\end{itemize}
\end{multicols}

Transformers take strings as inputs. Given a string $\str=\sym_1\cdots\sym_\length$ over an alphabet $\alphabet$, we append a special end-of-sentence symbol $\eos \notin \alphabet$ to form the extended string $\eosstr \defeq \sym_1 \cdots \sym_\length \eos$ over the extended alphabet $\eosalphabet \defeq \alphabet \cup \{\eos\}$. 

The transformer consists of an input layer, followed by a stack of $\layernumber$ hidden layers and a final output layer. 

\subsubsection{Input layer}
An embedding function $\ve$ of type $\eosalphabet\rightarrow\fpnset^\dimension$ maps each symbol to a $\dimension$-dimensional column vector. The embedding layer applies $\ve$ to each symbol in $\eosstr$ to produce the input representation:
$\embedding(\eosstr)\in\fpnset^{\dimension\times(\length+1)}$, i.e.,
\begin{equation}
\label{eq:input}
\embedding(\eosstr)_{:,n} \defeq \ve(\sym_n),  \quad n \in\{1,\ldots,\length+1\}.
\end{equation}
where $\sym_{\length+1} \defeq \eos$.

\subsubsection{Hidden layers}
Each hidden layer contains two sublayers: a self-attention mechanism and a feedforward network.

The self-attention mechanism is a function of type $\fpnset^{\dimension\times(\length+1)}\rightarrow\fpnset^{\dimension\times(\length+1)}$.
 Given input $\transformer(\eosstr) \in \fpnset^{\dimension \times (\length+1)}$, we compute:
\begin{subequations}
\label{eq:qkv}
\begin{align}
    \query &\defeq \W[Q]\transformer(\eosstr), \\
    \key &\defeq \W[K]\transformer(\eosstr), \\
    \val &\defeq \W[V]\transformer(\eosstr), 
\end{align}
\end{subequations}
where $\W[Q], \W[K], \W[V]\in\fpnset^{\dimension\times\dimension}$ are learnable parameter matrices.

A pairwise compatibility score $\score \in \fpnset^{(\length+1) \times (\length+1)}$ is computed via scaled dot product:
\begin{equation}
\label{eq:score}
\score_{n,m} \defeq \frac{\query[:][n]\bigcdot \key[:][m]}{\sqrt{\dimension}}.
\end{equation}
We then compute attention weights $\aalpha \in \fpnset^{(\length+1) \times (\length+1)}$ using $\softmax$:
\begin{equation}
\label{eq:softmax_SA}
\aalpha_{n,m} \defeq \frac{\exp(\score_{n,m})}{\sum_{i<n} \exp(\score_{n,i})}.
\end{equation}
In practice, the $\softmax$ is numerically stabilized by subtracting the maximum score from all scores before exponentiation:
\begin{equation}
    \label{eq:softmax_stable}
    \aalpha_{n,m} \defeq \frac{\exp(\score_{n,m} - \max_{i<n}\score_{n,i})}{\sum_{j<n} \exp(\score_{n,j} - \max_{i<n}\score_{n,i})}.
\end{equation}
The attention output at position $n$ is then computed as a weighted sum of the value vectors:
\begin{equation}
\label{eq:attention}
    \attention(\transformer(\eosstr))_{:,n} \defeq \sum_{m<n} \aalpha_{n,m}\val[:][m].
\end{equation}
We assume a single attention head, since multiple heads do not increase expressive power \citep{li2024chain}.

The feedforward sublayer is a function of type $\fpnset^{\dimension\times(\length+1)}\rightarrow\fpnset^{\dimension\times(\length+1)}$ and consists of two linear transformations with a ReLU nonlinearity:
\begin{equation}
\ffn(\transformer(\eosstr))_{:,n} \defeq \W[F2]\ReLU\left(\W[F1]\transformer[][:][n][\eosstr]+\bb[F1]\right) + \bb[F2],
\end{equation}
where $\W[F1]\in\fpnset^{\dimension'\times\dimension},\bb[F1]\in\fpnset^{\dimension'},\W[F2]\in\fpnset^{\dimension'\times\dimension},\bb[F2]\in\fpnset^\dimension$ are trainable parameters, $\dimension'$ is the hidden size of the intermediate layer.

\subsubsection{Output layer}
The final output layer is a function of type $\fpnset^{\dimension\times(\length+1)}\rightarrow \fpnset$ that computes a scalar score based on the representation at the final position  (corresponding to $\eos$): 
\begin{equation}
  \classification\left(\transformer[\layernumber][][][\eosstr]\right) = \left(\bw^{\text{C}}\right)^\top\transformer[\layernumber][:][\length+1][\eosstr] + b^{\text{C}}
\end{equation}
where $\bw^{\text{C}}\in\fpnset^{\dimension}$ and $b^{\text{C}}\in\fpnset$ are learnable parameters.

\subsection{From transformers to $\pfo$}
\label{app:transformer2pfo}
In this sub-section, we show that any transformer can be simulated by $\pfo$.
We begin by formalizing what it means for a function to be simulated by $\pfo$.
With a slight abuse of notation, we write $\dimension$ to refer either to the model’s hidden dimensionality or to 1, in the case of a scalar classification output.
\begin{definition}
    A function $\transformer \colon \eosalphabet^{\length+1} \to \fpnset^{\dimension \times (\length+1)}$ is said to be simulated by $\pfo$ if for every dimension $d \in \{1,\ldots,\dimension\}$ and every floating-point value $\fpn \in \fpnset$, there exists a single-variable formula $\fof(x)$ in $\pfo$ such that for every position $n \in \{1,\ldots,\length+1\}$,
  \begin{equation}
      \transformer[][d][n][\eosstr]=\fpn \quad \text{ if and only if } \quad \eosstr, n\models \fof(x).
  \end{equation}

  Similarly, a function $\transformer$ of type $\eosalphabet^{\length+1}\rightarrow\fpnset^{(\length+1)\times (\length+1)}$ is said to be simulated by $\pfo$ if for every floating-point value $\fpn \in \fpnset$, there exists a two-variable formula $\fof(x, y)$ in $\pfo$ such that for every pair of positions $n, m \in \{1,\ldots,\length+1\}$,
  \begin{equation}
      \transformer[][n][m][\eosstr]=\fpn \quad \text{ if and only if } \quad \eosstr, n, m\models \fof(x, y).
  \end{equation}
\end{definition}

We will prove the following proposition, which we will use repeatedly.
\begin{proposition}
\label{prop:ff}
Let $\transformer_1, \transformer_2$ be functions of type $\eosalphabet^{\length+1} \rightarrow \fpnset^\dimension$, and let $\transformer$ be either a function $\fpnset^\dimension \to \fpnset^{\dimension'}$ (where $\dimension'$ is not necessarily equal to $\dimension$) or $\fpnset^\dimension \times \fpnset^\dimension \to \fpnset$.
If both $\transformer_1$ and $\transformer_2$ can be simulated by $\pfo$, then so can $\transformer(\transformer_1)$ and $\transformer(\transformer_1, \transformer_2)$.
\end{proposition}
\begin{proof}
    Since the inputs are fixed-dimensional vectors over $\fpnset$, the number of possible value combinations is finite. For any given output value $\fpn \in \fpnset$, one can enumerate all input configurations for which the function yields $\fpn$, and construct a $\pfo$ formula expressing their disjunction.
\end{proof}

This proposition accounts for all components of the transformer architecture, such as the feedforward sublayers, projection operations, dot products, elementwise operations such as addition, scalar operations such as division, the $\exp$ function, and the classification head---except for the embedding layer, the summations involved in the attention computation, and the operation of identifying the maximum score introduced by the stabilized softmax.

We handle the embedding layer as follows:
\begin{lemma}
\label{lemma:embedding2pfo}
The embedding layer can be simulated by $\pfo$.
\end{lemma}
\begin{proof}
For any dimension $d$ and position $n$, the embedding layer outputs a fixed value $\fpn$ if the symbol at position $n$ is mapped to $\fpn$ by the embedding function $\ve$ in the $d$-th coordinate. Since $\eosalphabet$ is finite and $\ve$ is fixed, we can express this using a disjunction over all symbols that are mapped to $\fpn$ in coordinate $d$:\looseness=-1
    \begin{equation}
        \embedding(\eosstr)_{d,n} = \fpn \quad \text{ if and only if } \quad \eosstr,n\models\bigvee_{\syma\in\{\syma\mid\ve(\syma)_d=\fpn\}}\atom_\syma(x).
    \end{equation}  
\end{proof}

We next address the summations in the attention mechanism.
We begin by establishing the following result:\looseness=-1
\begin{restatable}{reProposition}{threshcount}
\label{prop:count}
$\pfo$ can count up to a threshold.  
\end{restatable}
\begin{proof}
The threshold counting quantifiers can be defined as follows.
\begin{itemize}[leftmargin=*]
    \item The quantifier for ``there exists at least one'' coincides with the standard bounded existential quantifier:
    \begin{equation}
    \exists^{\geq 1} y<x: \fof(y) \defeq \exists y<x: \fof(y).
    \end{equation}
    \item ``There exist at least two'' can be defined as:
    \begin{equation}
        \exists^{\geq 2} y<x: \fof(y) 
        \defeq \exists y<x: (\fof(y) \land \exists x<y: \fof(x)).
    \end{equation}
    \item  We can define ``exactly one'' by combining the above: 
    \begin{equation}
        \exists^{=1} y<x: \fof(y) \\
        \defeq \exists^{\geq 1} y<x: \fof(y) \land \lnot \exists^{\geq 2} y<x: \fof(y).
    \end{equation}
    \item Additional counting quantifiers up to a threshold can be defined in a similar fashion.
\end{itemize}
\end{proof}

Soft attention involves two types of summations, which we address in turn.
The first appears in the denominator of the $\softmax$ function, used to compute the attention weights $\aalpha$ (\cref{eq:softmax_SA}).
Since $\exp$ outputs non-negative values, this summation consists entirely of non-negative numbers and can be simulated by $\pfo$.

\begin{restatable}{reLemma}{sumpos}
    \label{lemma:sum_pos}
    $\pfo$ can simulate a sum of non-negative fixed-precision floating-point numbers.     
\end{restatable}
\begin{proof}
Since $\nnfpnset$ is finite, there exists a finite set of possible sums.
Moreover, because all values are non-negative, for each such sum
there are only finitely many combinations of input values that yield it.
Each combination can be characterized using the threshold-counting
quantifiers introduced in \cref{prop:count}:
Too many positive entries lead to overflow, while $0$'s do not affect the sum.
Taking a finite disjunction over all such combinations
defines a $\pfo$ formula that holds exactly when
the prefix sum equals a given value.
\end{proof}

The second summation occurs in the computation of attention outputs (\cref{eq:attention}).
The value matrix $\val\in\fpnset^{\dimension\times (\length+1)}$ may include both positive and negative values, so we cannot apply the previous approach directly. However, as pointed out by \citet{NEURIPS2023_a48e5877}:
\begin{restatable}{reProposition}{attentionspan}
\label{prop:softmax_finite}
Under fixed precision, there exists an upper limit on the number of non-zero entries in the output of a $\softmax$ function. We refer to the upper bound as the \emph{maximum attention span} $\nmax$.\looseness=-1
\end{restatable}
\begin{proof}
This upper bound is given by:
\begin{equation}
    \nmax = \left\lfloor\frac{\min(1,\max(\fpnset\setminus\{\infty\}))}{\min(\fpnset_{>0})}\right\rfloor,
\end{equation}
where $\lfloor\cdot\rfloor$ is the floor function.
\end{proof}
Consequently, the number of nonzero attention weights is bounded. It follows that \cref{eq:attention} computes the sum over a bounded number of terms, which can be simulated by $\pfo$:
\begin{restatable}{reLemma}{sumbound}
\label{lemma:sum_bounded}
$\pfo$ can simulate a sum of a bounded number of fixed-precision floating-point numbers. 
\end{restatable}
\begin{proof}
Since the number of summands is bounded, we can enumerate all possible combinations of inputs. This behavior can be simulated using the threshold counting quantifiers introduced in \cref{prop:count}.
\end{proof}

Finally, we account for the numerical-stability adjustment in \cref{eq:softmax_stable}.
The only additional operation introduced by this stabilization is the identification of the maximum score among a finite set of floating-point values, which can be expressed in $\pfo$ as follows:
\begin{proposition}
    \label{prop:max}
    $\pfo$ can identify the maximum score $\max_{i<n}\score_{n,i}$.
\end{proposition}
\begin{proof}
    Let the ordered set of floating-point values be 
    $\fpnset = \{-\infty, \fpn_1, \fpn_2, \ldots, \fpn_K, \infty\}$. 
    For each $\fpn\in\fpnset$, let $\fof_\fpn(x)$ be a $\pfo$-formula that holds
    iff the value at position $x$ equals $\fpn$. 
    Since the set of floating-point numbers is finite, we can enumerate all possible maximum values. For each candidate maximum value $\fpn$, we can construct a $\pfo$ formula $\fof_\fpn^{\max}(x)$ that holds if and only if the scores among all positions $y < x$ are less than or equal to $\fpn$, and at least one value equals $\fpn$. 
    They can be defined inductively over the total order of $\fpnset$ as follows:
    \begin{subequations}
        \begin{align}
            \fof_\infty^{\max}(x) &\defeq \exists y<x: \fof_\infty(y), \\
            \fof_{\fpn_K}^{\max}(x) &\defeq \left(\lnot \exists y<x: \fof_\infty^{\max}(y)\right) \land \exists y<x: \fof_{\fpn_K}(y), \\
            &\;\;\vdots \nonumber \\
            \fof_{\fpn_1}^{\max}(x) &\defeq \left(\lnot \exists y<x: \bigvee_{\fpn\in\{\infty,\fpn_K,\ldots,\fpn_2\}}\fof_\fpn^{\max}(y)\right) \land \exists y<x: \fof_{\fpn_1}(y), \\
            \fof_{-\infty}^{\max}(x) &\defeq \left(\lnot \exists y<x: \bigvee_{\fpn\in\{\infty,\fpn_K,\ldots,\fpn_1\}}\fof_\fpn^{\max}(y)\right) \land \exists y<x: \fof_{-\infty}(y).
        \end{align}
    \end{subequations}
\end{proof}

We have now completed the proof of the following theorem:
\ttopfo*
\begin{proof}
By \cref{lemma:embedding2pfo}, the embedding layer can be simulated by $\pfo$.
By \cref{prop:ff}, simulation by $\pfo$ is closed under all subsequent operations in the transformer, except for summations.
By \cref{lemma:sum_pos}, \cref{prop:softmax_finite}, and \cref{lemma:sum_bounded}, the summation operations involved in attention computation can also be simulated by $\pfo$.
Therefore, the entire transformer can be simulated by $\pfo$.
\end{proof}

\subsection{From $\ptl$ to transformers}
\label{app:ptl2transformer}
In this subsection, we show that every $\ptl$ formula can be simulated by a transformer. The central idea is to encode the truth value of the formula at each position as a dedicated coordinate within the transformer's hidden state: the value $1$ represents $\true$, and $0$ represents $\false$.

The formal definition is as follows.
\begin{definition}
    A $\ptl$ formula $\tlf$ is said to be simulated by a function $\transformer$ of type $\eosalphabet^{\length+1}\rightarrow\fpnset^{\dimension\times(\length+1)}$ if there exists a coordinate $d\in\{1,\ldots,\dimension\}$ such that for every position $n\in\{1,\ldots,\length+1\}$ and every input string $\eosstr\in\kleene{\eosalphabet}$,
    \begin{equation}
        \transformer[][d][n][\eosstr] =
        \begin{cases}
            1, &\text{if } \eosstr,n\models \tlf, \\
            0, &\text{otherwise.}
        \end{cases}
    \end{equation}
\end{definition}

We introduce a one-hot encoding function $\onehot$, which maps an index to a column vector of zeros except for a single entry equal to $1$ at that index. Let $\zero$ denote the all-zero column vector.

We first show that the atomic formulas can be translated into an embedding layer.
\begin{lemma}
\label{lemma:atomic2transformer}
    The atomic formulas $\atom_\syma$ for $\syma\in\eosalphabet$ can be simulated by an embedding layer.
\end{lemma}
\begin{proof}
    We assign a distinct coordinate $d$ to every distinct symbol $\syma\in\eosalphabet$ with a map $\ltod\colon \eosalphabet\rightarrow\{1,\ldots,\eosalphabetsize\}$. Since the biases in attention mechanisms are sometimes omitted, we include a dedicated bias coordinate $\coord{b}$ that always holds a constant value $1$. The embedding layer is then defined as follows:
    \begin{equation}
        \transformer[0][d][n][\eosstr] = 
        \begin{cases}
            1,  & \text{if } d = \coord{b}, \\
            1,  & \text{if } d = \ltod(\sym_n), \\
            0,  & \text{otherwise.}
        \end{cases}
    \end{equation}
    Coordinates that remain inactive at this stage are initialized to $0$; they will later be used to store logical formulas introduced in subsequent layers.
\end{proof}

Next, we show that the Boolean connectives can be simulated.
\begin{lemma}
\label{lemma:boolean2transformer}
    If $\tlf_1$ and $\tlf_2$ can be simulated by a transformer, then their Boolean combinations can also be simulated by a transformer.
\end{lemma}
\begin{proof}
    Assume $\tlf_1$ and $\tlf_2$ are simulated by $\transformer[\ell][d_1][:]$ and $\transformer[\ell][d_2][:]$ respectively. For brevity, we write $\vx=\transformer[\ell][d_1][:][\eosstr]$ and $\vy=\transformer[\ell][d_2][:][\eosstr]$, so $\vx,\vy \in \fpnset^{1\times(\length+1)}$ are row vectors indexed by positions. All operations below are applied elementwise.
    \begin{itemize}[leftmargin=*]
    \item $\lnot \tlf_1$: Construct a layer that computes, in a fresh coordinate $\coord{\lnot}$,
    \begin{equation}
        -\vx + 1.
    \end{equation}
    This can be implemented with an FFN as follows. 
    Select a hidden unit $d'$ in the intermediate hidden state $\{1,\ldots,\dimension'\}$ and set the first linear transformation to:
    \begin{equation}
    \W[F1]_{d,:} = 
    \begin{cases}
        -\onehot[d_1]^\top & \text{if } d=d',\\
        \zero^\top & \text{otherwise,}
    \end{cases}
    \qquad
    \bb[F1] = \onehot[d'].
    \end{equation}
    The second linear transformation maps $d'$ to the output coordinate $\coord{\lnot}$:
    \begin{equation}
    \W[F2]_{d,:} = 
    \begin{cases}
        \onehot[d']^\top & \text{if } d=\coord{\lnot},\\
        \zero^\top & \text{otherwise,}
    \end{cases}
    \qquad
    \bb[F2] = \zero.
    \end{equation}
    The attention sublayer is bypassed by setting $\W[Q],\W[K],\W[V]$ to zero matrices.
    \item $\tlf=\tlf_1\land \tlf_2$: 
    Compute, in a fresh coordinate $\coord{\land}$,
    \begin{equation}
        \ReLU\left(\vx+\vy-1\right).
    \end{equation}
    Implement this with a single FFN as follows. 
    Select a hidden unit $d'$ and set the first linear transformation to:
    \begin{equation}
    \W[F1]_{d,:} = 
    \begin{cases}
        \onehot[d_1]^\top + \onehot[d_2]^\top & \text{if } d=d',\\
        \zero^\top & \text{otherwise,}
    \end{cases}
    \qquad
    \bb[F1] = -\onehot[d'].
    \end{equation}
    The second linear transformation connects $d'$ to the output coordinate $\coord{\land}$:
    \begin{equation}
    \W[F2]_{d,:} = 
    \begin{cases}
        \onehot[d']^\top & \text{if } d=\coord{\land},\\
        \zero^\top & \text{otherwise,}
    \end{cases}
    \qquad
    \bb[F2] = \zero.
    \end{equation}
    Again, the attention sublayer is bypassed. 
    \end{itemize}
\end{proof}

Finally, we address the temporal operator $\past$. 
\begin{lemma}
\label{lemma:past2transformer}
    If $\tlf$ can be simulated by a transformer, then $\past\tlf$ can also be simulated by a transformer.
\end{lemma}
\begin{proof}
    We start by introducing the base construction and then discuss two issues arising from fixed-precision arithmetic, together with their corresponding fixes. 
    Assume $\tlf$ is simulated by $\transformer[\ell][d_1][:]$.
    \paragraph{Base construction.}
    We construct an attention mechanism such that each position $n$ attends uniformly to all previous positions $m < n$ where $\transformer[\ell][d_1][m][\eosstr] = 1$. Specifically, we need the attention scores to be:
    \begin{equation}
    \score_{n,m} = 
    \begin{cases}
        0 & \text{if } \transformer[\ell][d_1][m][\eosstr] = 1, \\
        -\fpnlarge & \text{if } \transformer[\ell][d_1][m][\eosstr] = 0, \\
    \end{cases}
    \end{equation}
    where $\fpnlarge\in\fpnset$ is a large enough positive floating-point number such that $\exp(-\fpnlarge)=0$.
    This ensures that positions where $\tlf$ is $\false$ receive zero attention weight.

    This behavior can be implemented by configuring the query and key vectors as follows:
    \begin{equation}
        \forall n, \quad \query[d][n] = 
        \begin{cases}
            \fpnlarge &\text{if } d = d_1, \\
            0 &\text{otherwise,}
        \end{cases}
    \end{equation}
    and 
    \begin{equation}
        \key=\transformer[\ell][][][\eosstr] - 1.
    \end{equation}
    where the subtraction by $1$ is applied elementwise.

    The corresponding attention projection weights are:
    \begin{equation}
    \W[Q]_{d,:} = 
    \begin{cases}
        \fpnlarge\onehot[\coord{b}]^\top & \text{if } d=d_1, \\
        \zero^\top & \text{otherwise,}
    \end{cases}
    \end{equation}
    and 
    \begin{equation}
        \W[K]_{d,:} = \onehot[d]^\top - \onehot[\coord{b}]^\top,
    \end{equation}
    where $\coord{b}$ denotes the bias coordinate.

    Next, we copy the values from $\transformer[\ell][d_1][:]$ into an unused coordinate $\coord{\past}$ via the value projection:
    \begin{equation}
    \W[V]_{d,:} = 
    \begin{cases}
        \onehot[d_1]^\top & \text{if } d=\coord{\past}, \\
        \zero^\top & \text{otherwise.}
    \end{cases}
    \end{equation}
    Thus, $\val$ becomes:
    \begin{equation}
    \val[d][:] = 
    \begin{cases}
        \transformer[\ell][d_1][:][\eosstr] & \text{if } d=\coord{\past}, \\
        \zero^\top & \text{otherwise.}
    \end{cases}
    \end{equation}
    Let $\nvalid$ denote the number of positions $m < n$ such that $\transformer[\ell][d_1][m][\eosstr] = 1$.
    If $\nvalid > 0$, then $\aalpha_{n,m} = 1 / \nvalid$, giving $\transformer[\ell+0.5][\coord{\past}][n][\eosstr] = 1$.
    Otherwise, if $\nvalid = 0$, the attention weights are distributed uniformly over previous positions, all of which satisfy $\transformer[\ell][d_1][m][\eosstr] = 0$, yielding $\transformer[\ell+0.5][\coord{\past}][n][\eosstr] = 0$.
    Hence, $\transformer[\ell+0.5][\coord{\past}][:]$ simulates $\past \tlf$ in the ideal case.

    \paragraph{Patch 1: Vanishing Attention Weights}
    When too many previous positions with $\transformer[\ell][d_1][m]=1$, i.e., $\nvalid>\nmax$, the resulting attention weights $\aalpha_{n,:}$ will underflow to $\zero^\top$. This causes $\transformer[\ell+0.5][\coord{\past}][n][\eosstr] = 0$ even when $\past \tlf$ is true, leading to incorrect behavior.

    To fix this, we first compute the logical conjunction of $d_1$ and $\coord{\past}$ and store the result in $\coord{\land}$ via the feedforward sublayer. When $\transformer[\ell+1][\coord{\land}][n] = 1$, it indicates that the position $n$ satisfies $\tlf$ and there exists at least one prior position that also satisfies $\tlf$. The number of such positions is bounded by $\nmax$, as starting from the $(\nmax+1)^{\text{th}}$ position, the first attention sublayer will suffer from vanishing attention weights, causing $\transformer[\ell+0.5][\coord{\past}][n][\eosstr] = 0$.

    Next, we construct a second layer in which the attention sublayer uniformly attends to positions $m < n$ where $\transformer[\ell+1][\coord{\land}][m][\eosstr] = 1$. This constructs a coordinate $\coord{\past\past}$ simulating $\past(\tlf\land\past\tlf)$. Since the number of positions where $\transformer[\ell+1][\coord{\land}][m] = 1$ is bounded, this second attention sublayer is not subject to vanishing attention weights and reliably computes its output.
    
    Note that $\past(\tlf\land\past\tlf)$ evaluates to $\true$ when at least two earlier positions satisfy $\tlf$. If exactly one earlier position satisfies $\tlf$, then $\past\tlf$ holds but $\past(\tlf\land\past\tlf)$ does not. In this case, however, the first attention sublayer (producing $\coord{\past}$) already correctly simulates $\past\tlf$,  since the attention only needs to cover one position. 
    Thus, we take the logical disjunction of $\coord{\past}$ and $\coord{\past\past}$, implemented via the feedforward sublayer. The resulting coordinate $\coord{\lor}$ correctly simulates the formula $\past\tlf$. 

    \paragraph{Patch 2: Rounding Errors}
    Fixed-precision arithmetic introduces rounding errors when computing the attention weights $\aalpha_{n,m}$. The fraction $1/\nvalid$ may not be exactly representable by a finite-precision floating-point number. Moreover, the iterative summation in \cref{eq:attention} can further accumulate rounding errors. Consequently, the output $\transformer[\ell+0.5][\coord{\past}][n]$ may equal a value very close to $1$, when it should be exactly $1$. This could lead to incorrect results in subsequent computations.

    We address this with a thresholding mechanism implemented via a feedforward sublayer. 
    Define $\round{1}$ as the set of floating-point numbers that deviates from $1$ by at most $\delta$:
    \begin{equation}
    \label{eq:rounding}
    \round{1} = \left\{ \fpn \in \fpnset \mid |\fpn - 1| \leq \delta \right\}.
    \end{equation}
    In standard floating-point systems, such as IEEE 754 \citep{4610935}, $\delta$ is on the order of the machine precision, e.g., approximately $2^{-24}$ for single precision and $2^{-11}$ for half precision. Let $\fpnsmall$ be a small positive floating-point number such that $1/\fpnsmall$ is also representable, $\fpnsmall\cdot\frac{1}{\fpnsmall}=1$, and $\min(\round{1})-\fpnsmall > 0$. We rectify the value at $\coord{\past}$ using:
    Let $\fpnsmall$ be a small positive floating-point number such that $1/\fpnsmall$ is also representable, $\fpnsmall\cdot\frac{1}{\fpnsmall}=1$, and $\min(\round{1})-\fpnsmall > 0$. We rectify the value at $\coord{\past}$ using:
    \begin{equation}
        \label{eq:threshold}
        -\frac{1}{\fpnsmall}\left(\ReLU\left(\transformer[\ell+0.5][\coord{\past}][:] - \fpnsmall\right) - \ReLU\left(\transformer[\ell+0.5][\coord{\past}][:]\right)\right).
    \end{equation}
    This function outputs $1$ if $\transformer[\ell+0.5][\coord{\past}][n]\in \round{1}$, and $0$ if $\transformer[\ell+0.5][\coord{\past}][n]=0$. 
    To implement this with an FFN, select two hidden units $d'_1, d'_2$ and define:
    \begin{equation}
        \W[F1]_{d,:} = 
        \begin{cases}
            \onehot[\coord{\past}]^\top & \text{if } d=d'_1 \text{ or } d=d'_2,\\
            \zero^\top & \text{otherwise,}
        \end{cases}
        \qquad
        \bb[F1] = -\fpnsmall\onehot[d'_1].
    \end{equation}
    and
    \begin{equation}
        \W[F2]_{d,:} = 
        \begin{cases}
            -\frac{1}{\fpnsmall}(\onehot[d'_1]^\top+\onehot[d'_2]^\top) & \text{if } d=\coord{\past},\\
            \zero^\top & \text{otherwise,}
        \end{cases}
        \qquad
        \bb[F2] = \zero.
    \end{equation}
\end{proof}

So far, we have ignored layer normalization. We now show that its presence does not affect the representational power of transformers.
\begin{proposition}
\label{prop:ln}
    If a $\ptl$ formula $\tlf$ can be simulated by a transformer without layer normalization, then it can also be simulated by a transformer with layer normalization.
\end{proposition}
\begin{proof}
    Let $\vx$ be a vector. Layer normalization is defined as follows:
    \begin{equation}
        \text{LayerNorm}(\vx) = \frac{\vx - \mu}{\sqrt{\sigma^2+\epsilon}} \cdot \gamma + \beta,
    \end{equation}
    where $\mu$ and $\sigma$ are the mean and standard deviation of the elements in $\vx$, $\gamma$ and $\beta$ are parameters, and $\epsilon$ is a small constant for numerical stability.

    We introduce, for every logical coordinate $d$, a mirror coordinate $\mir{d}$ and enforce
    \begin{equation}
        \transformer[\ell][\mir{d}][n][\eosstr] = 1 - \transformer[\ell][d][n][\eosstr]
    \end{equation}
    for all sublayers $\ell$ and positions $n$. This guarantees that for every position $n$, 
    \begin{equation}
        \mu_n = \frac{1}{2}\quad \text{and}\quad \sigma_n^2 = \frac{1}{4}.
    \end{equation}
    independently of the particular truth assignment.

    With this structure, we can choose the parameters of layer normalization such that the operation becomes the identity function (i.e., has no effect):
    \begin{equation}
        \gamma = \sqrt{\sigma_n^2+\epsilon} = \sqrt{\frac{1}{4}+\epsilon}, \quad \beta = \mu_n  = \frac{1}{2}.
    \end{equation}

    However, the attention sublayer may produce values close to $1$ but not exactly equal to it. Consequently, layer normalization can slightly perturb these values away from the intended Boolean values. We denote by $\round{0}$ and $\round{1}$ the sets of floating-point numbers that deviate from $0$ and $1$, respectively, by at most a small constant $\delta$, which is again on the order of the machine precision. This deviation is harmless, since the subsequent feedforward sublayer can rectify the outputs using the following thresholding function: let $d_1$ be the coordinate we hope to rectify and choose $\fpn_1,\fpn_2\in\fpnset$ with $\max(\round{0}) < \fpn_1 < \fpn_2 < \min(\round{1})$. Apply:
    \begin{equation}
        \frac{1}{\fpn_2-\fpn_1}\left(\left(\transformer[\ell][d_1][:][\eosstr] - \fpn_1\right) - \ReLU\left(\transformer[\ell][d_1][:][\eosstr]-\fpn_2\right)\right).
    \end{equation}
    To implement this with an FFN, select two hidden units $d'_1, d'_2$ and define:
    \begin{equation}
        \W[F1]_{d,:} = 
        \begin{cases}
            \onehot[d_1]^\top & \text{if } d=d'_1 \text{ or } d=d'_2,\\
            \zero^\top & \text{otherwise,}
        \end{cases}
        \qquad
        \bb[F1] = -\fpn_1\onehot[d'_1] - \fpn_2\onehot[d'_2].
    \end{equation}
    and 
    \begin{equation}
        \W[F2]_{d,:} = 
        \begin{cases}
            \frac{1}{\fpn_2-\fpn_1}(\onehot[d'_1]^\top - \onehot[d'_2]^\top) & \text{if } d=d_1,\\
            \zero^\top & \text{otherwise,}
        \end{cases}
        \qquad
        \bb[F2] = \zero.
    \end{equation}
    This yields $1$ if $\transformer[\ell][d_1][n][\eosstr]\in \round{1}$, and $0$ if $\transformer[\ell][d_1][n][\eosstr]\in \round{0}$.
\end{proof}

Now, we obtain the following theorem:
\ptltot*
\begin{proof}
    By applying structural induction to the formula, using the lemmas \cref{lemma:atomic2transformer}, \cref{lemma:boolean2transformer}, and \cref{lemma:past2transformer}, we demonstrate that each formula $\tlf$ can be simulated by dimension $d_\tlf$ in the transformer's hidden state. Furthermore, by \cref{prop:ln}, we can nullify the effect of layer normalization.
    
    Finally, the classification head only needs to copy the relevant dimension $d_\tlf$ as the output.
\end{proof}

\section{Characterizations of $\ptl$}
\label{app:characterizations}
In this section, we introduce three alternative formalisms and show that they are equivalent in expressive power to $\ptl$.

\subsection{Left-deterministic polynomials}
A \defn{monomial} over an alphabet $\alphabet$ is a language of the form:
\begin{equation}
\alphabet_0^* \syma_1 \alphabet_1^*\cdots \syma_n\alphabet_n^*,
\end{equation}
where $\syma_1,\ldots,\syma_k\in\alphabet$ and $\alphabet_0,\alphabet_1,\ldots,\alphabet_n\subseteq \alphabet$.

A monomial is called 
\begin{itemize}
    \item \defn{left-deterministic} if for every $k\in\{1,\ldots,n\}$, $\syma_k\not\in\alphabet_{k-1}$,
    \item \defn{right-deterministic} if for every $k\in\{1,\ldots,n\}$, $\syma_k\not\in\alphabet_{k}$,
    \item \defn{unambiguous} if it is either left-deterministic or right-deterministic.
\end{itemize}

A \defn{polynomial} is a finite union of monomials, and it is said to be left-deterministic (resp. right-deterministic or unambiguous) if it is a finite disjoint union of left-deterministic (resp. right-deterministic or unambiguous) monomials. 

\begin{testexample}
    The language \textsc{first} $\symb\kleene{\alphabet}$ is left-deterministic and therefore definable in $\ptl$. In contrast, the language \textsc{last} $\kleene{\alphabet}\symb$ is right-deterministic but not left-deterministic and thus not recognizable by $\ptl$.
\end{testexample}

\subsection{$\gR$-trivial monoids}
A \defn{semigroup} is a set endowed with an associative operator, and a \defn{monoid} is a semigroup additionally endowed with an identity element.
A canonical example is the set $\kleene{\alphabet}$, with string concatenation as the associative operation and the empty string $\varepsilon$ as the identity element.
A monoid $\monoid$ is said to be $\gR$-\defn{trivial} (resp. $\gL$-\defn{trivial}) if for all elements $s,t\in \monoid$, the condition $s\monoid = t\monoid$ (respectively, $\monoid s = \monoid t$) implies $s = t$.

We now define an equivalence relation on $\kleene{\alphabet}$ known as the \defn{syntactic congruence}. Given a language $\lang \subseteq \kleene{\alphabet}$, two strings $\strs, \strt \in \kleene{\alphabet}$ are \defn{syntactically equivalent}, written $\strs \equivalent \strt$, if and only if:
\begin{equation}
\forall \stru,\strv\in\kleene{\alphabet}\colon \stru\strs\strv\in\lang\Leftrightarrow\stru\strt\strv\in\lang.
\end{equation}
The equivalence class of a string $\str$ is denoted as $\equivClass{\str}$. 
The quotient monoid $\kleene{\alphabet} / \equivalent$, i.e., the set of equivalence classes under $\equivalent$, is called the \defn{syntactic monoid} of the language $\lang$.

\subsection{Partially ordered DFAs}
\label{app:po}
PODFA is introduced in \cref{sec:characterizations}.
Constructing the DFA for a given language is a useful method for assessing whether the language is definable in $\ptl$.
\begin{testexample}
    It is non-trivial to determine whether certain languages can be defined in $\ptl$. However, it becomes clearer when we consider its corresponding automaton.
    For instance, consider the language \textsc{dyck}-$(1,1)$, written as $\kleene{(\syma\symb)}$ in regular expression, the Dyck language over one type of parentheses with nesting depth limited to $1$. The minimal DFA accepting this language is shown below:
    \begin{center}
        \begin{tikzpicture}
            \node[state, initial] (q0) [] { $\stateq_0$ }; 
            \node[state, accepting] (q1) [right = of q0] { $\stateq_1$ }; 
            \draw[transition] (q0) edge[auto, bend left] node{$\syma$} (q1) 
            (q1) edge[auto, bend left] node{$\symb$} (q0);
        \end{tikzpicture}  
    \end{center}
    This DFA is not partially ordered, as it can revisit both $\stateq_1$ and $\stateq_0$. Therefore, \textsc{Dyck}-$(1,1)$ is not definable in $\ptl$.
\end{testexample}

\subsection{Equivalence}
\begin{table*}
\centering
\footnotesize
\caption{Relations between logics, formal languages, monoids, and finite automata.}
\label{tab:equivalence}
\begin{tabular}{ccccc}
    Temporal logic & First-order logic & Formal language & Monoid & Automata \\
    \midrule
    $\ltl$ & $\fo$ & Star-free & Aperiodic & Counter-free \\
    $\tl$ & $\fo[2]$ & Unambiguous polynomial & DA & $2$-PODFA \\
    $\ptl$ & $\pfo$ & Left-deterministic polynomial & $\gR$-trivial & PODFA
\end{tabular}
\end{table*}

We first show that every PODFA can be defined in $\ptl$.
\begin{lemma}
    \label{lemma:po2ptl}
    Let $\automaton=\dfatuple$ be a PODFA and $\str\in\alphabet$ a string of length $\length$. For every $\stateq\in\states$, there exists a formula $\ltlf{\stateq}$ in $\ptl$ such that $\automaton$ is in $\stateq$ upon reading $\str_{<n}$ if and only if $\str,n-1\models\ltlf{\stateq}$.
\end{lemma}
\begin{proof}
    We first construct a formula $\stlf{\stateq}$ for each state $\stateq\in\states$, such that $\str, n-1 \models \stlf{\stateq}$ if and only if $\automaton$ is in $\stateq$ after reading $\str_{<n-1}$, i.e., \emph{right before} reading $\sym_{n-1}$. We proceed by induction on the partial order $\preceq$ defined over $\states$.
    \paragraph{Base case.} The initial state $\qinit$ can be defined by ensuring that no prior symbol caused a transition out of $\qinit$:
    \begin{equation}
        \stlf{\qinit} = \bigwedge_{\syma\in\{\syma\mid \trans(\qinit, \syma) \neq \qinit\}}\lnot \past \atom_{\syma}.
    \end{equation}
    \paragraph{Induction step.} Assume that for all $\stateq' \preceq \stateq$ with $\stateq' \neq \stateq$, the formulas $\stlf{\stateq'}$ have already been constructed. Then, $\automaton$ has entered $\stateq$ at some prior point if:
    \begin{equation}
        \etlf{\stateq} = \bigvee_{\stateq'\in\states}\bigvee_{\syma\in\{\syma\mid \trans(\stateq', \syma) = \stateq\}} \past (\stlf{\stateq'} \land \atom_{\syma}).
    \end{equation} 
    To define $\stlf{\stateq}$, we must ensure that $\stateq$ has been entered and it is not exited yet:
    \begin{equation}
        \stlf{\stateq} = \etlf{\stateq} \land \bigwedge_{\syma\in\{\syma\mid \trans(\stateq, \syma) \neq \stateq\}}\lnot \past (\etlf{\stateq} \land \atom_{\syma}).
    \end{equation}

    Once we obtain $\stlf{\stateq}$, we can define $\ltlf{\stateq}$ as follows:
    \begin{equation}
        \ltlf{\stateq} = \bigvee_{\stateq'\in\states}\bigvee_{\syma\in\{\syma\mid \trans(\stateq', \syma) = \stateq\}} \stlf{\stateq'} \land \atom_\syma
    \end{equation}
    Finally, a string $\str$ is accepted by $\automaton$ if it reaches a final state:
    \begin{equation}
        \bigvee_{\stateq\in\final} \tlf_{\stateq}.
    \end{equation}
\end{proof}

Now we move on to prove the direction from $\ptl$ to $\gR$-trivial monoid. Our proof is inspired by the proof of Proposition 4 in \citet{doi:10.1142/S0129054108005802}. 

A key characterization of $\gR$-trivial monoids is given by \citet{BRZOZOWSKI198032}:
\begin{proposition}[\citet{BRZOZOWSKI198032}]
    \label{prop:r_trivial}
    $\monoid$ is $\gR$-trivial if and only if there exists an integer $K>0$ such that for all $s,t\in \monoid$, we have $(st)^K s=(st)^K$.
\end{proposition}

Thus, we prove the following lemma as preparation:
\begin{lemma}
\label{lemma:ptl2rtprep}
    Let $\tlf$ be a formula in $\ptl$ with operator depth at most $K$ and $K>0$. For every $\stru,\strv,\strs,\strt\in\kleene{\alphabet}$, we have 
    \begin{equation}
        \stru(\strs\strt)^K\strs\strv\models \tlf \quad \text{ if and only if } \quad \stru(\strs\strt)^K\strv\models \tlf. 
    \end{equation}
\end{lemma}
\begin{proof}
    The case $\strs = \varepsilon$ is trivial, so we assume $\strs \neq \varepsilon$.
    
    Let $\str=\stru(\strs\strt)^K\strs\strv$ and $\str'=\stru(\strs\strt)^K\strv$. As $\str'$ is a subsequence of $\str$, we align positions in $\str'$ with a subset of positions in $\str$. Define $\length=|\str|$. We consider position $\length + 1$ to point just beyond the end of $\str$.   

    We define a pair $(n,n')\in\{1,\ldots,\length+1\}$ as \defn{legal} if 
    \begin{equation}
        (n=n'=\length+1) \quad \lor \quad \left(\sym_{n} = \sym_{n'} \land (n'<\left|\stru(\strs\strt)^K\right| \lor n'>\left|\stru(\strs\strt)^K\strs\right|)\right).
    \end{equation}
    Next, define the middle block of $\str$ as:
    \begin{equation}
        \ball{k}=\left\{n\;\middle|\;\left|\stru(\strs\strt)^k\right|<n<\left|\stru(\strs\strt)^K\strs\right|\right\},
    \end{equation}
    A legal pair $(n, n')$ is said to be \defn{$k$-close} if:
    \begin{equation}
        n=n' \quad \lor \quad  n,n'\in\ball{k}.
    \end{equation}

    \medskip
    \noindent
    We now prove the following inductive claim: Let $\tlf$ be a formula in $\ptl$ with operator depth at most $k$, for some $0\leq k\leq K$. For every $k$-close pair $(n,n')$, we have 
    \begin{equation}
        \str,n\models \tlf \quad \text{ if and only if } \quad \str',n'\models \tlf.
    \end{equation}

    We proceed by induction on $k$.
    \paragraph{Base case.} $k=0$: In this case, $\tlf$ is a Boolean combination of atomic formulas $\atom_\syma$. $(n,n')$ being a legal pair implies either $n = n'=\length+1$ or $\sym_n = \sym_{n'}$, so the claim holds trivially.
    \paragraph{Induction step.} Assume the claim holds for depth $k-1$. 
    Let $\tlf_1=\past \tlf$ where $\tlf$ has depth $k-1$.
    Suppose $\str, n \models \tlf_1$. Then there exists a position $m < n$ such that $\str,m\models \tlf$.
    For $\str',n'\models \tlf_1$ to hold, we need to identify a position $m'< n'$ such that $\str',m'\models\tlf$. By the inductive hypothesis, it suffices for $m'<n'$ and for $(m,m')$ to be $(k-1)$-close.
    There are two cases:
    \begin{itemize}
        \item Case 1: $m\notin\ball{k}$. Then we can set $m'=m$. $(m,m')$ is $(k-1)$-close by definition. We now need to show that $m'<n'$. 
        \begin{itemize}
            \item Subcase 1.1: $n,n'\in\ball{k}$. Since $m\notin\ball{k}$ and $m<n$, we can conclude $m\leq \left|\stru(\strs\strt)^{k}\right|$. Thus, $m'=m<n'$ because $n'\in\ball{k}$.
            \item Subcase 1.2: $n=n'$. Then we have $m'=m<n=n'$.
        \end{itemize}
        \item Case 2: $m\in\ball{k}$. Then, as $\ball{k-1} \setminus \ball{k}$ covers all symbols in $\ball{k}$, there exists $m' \in \ball{k-1} \setminus \ball{k}$ such that $\sym_{m'} = \sym_m$. Again, $(m,m')$ is $(k-1)$-close since $m,m'\in\ball{k-1}$. We need to show that $m'<n'$:
        \begin{itemize}
            \item Subcase 2.1: $n,n'\in\ball{k}$. Since $m' \in \ball{k-1} \setminus \ball{k}$ and $n' \in \ball{k}$, we have $m' < n'$.
            \item Subcase 2.2: $n=n'$. Since $m' \in \ball{k-1} \setminus \ball{k}$ and $m \in \ball{k}$, we have $m'<m<n=n'$.
        \end{itemize} 
    \end{itemize}

    The converse direction, from $\str',n'\models \tlf$ to $\str,n\models \tlf$ follows analogously.
    \medskip
    \noindent
    
    By applying the claim with $k = K$, $n = \length+1$, and $n' = \length+1$, we conclude that $\str \models \tlf$ if and only if $\str' \models \tlf$.
\end{proof}

This result then follows straightforwardly:
\begin{lemma}
\label{lemma:ptl2rt}
    Every $\ptl$-definable language has a $\gR$-trivial syntactic monoid.
\end{lemma}
\begin{proof}
    \cref{lemma:ptl2rtprep} suffices to show that every $\ptl$-definable language satisfies the characterization in \cref{prop:r_trivial}.  
\end{proof}

We now have all the pieces in place to prove the main theorem.
\thmptl*
\begin{proof}
    \citet{BRZOZOWSKI198032} showed the equivalence of (1), (2), and (3). To complete the picture, it remains to incorporate (4), which follows from \cref{lemma:po2ptl} and \cref{lemma:ptl2rt}.
\end{proof}

The relationships between various temporal logics, first-order logics, formal languages, monoids, and finite automata are summarized in \cref{tab:equivalence}. We treat the last row (for $\ptl$) in this paper. For further details on the other classes ($\tl$ and $\ltl$), interested readers can refer to \citet{mcnaughton1971counter} and \citet{doi:10.1142/9789812776884_0021}.

\section{Transformer Language Models}
\label{app:tlm}
The transformer architecture follows the specification in \cref{app:transformer}, with one modification: the final output layer is replaced by a language modeling head $\lmhead\colon \fpnset^{\dimension}\rightarrow\fpnset^{|\eosalphabet\cup\{\unk\}|}$. $\lmhead$ is defined as follows:
\begin{equation}
    \lmhead(\transformer[\layernumber][][][\str_{<n}]) = \softmax\left(\W[L]\transformer[\layernumber][:][n-1][\str]+\bb[L]\right)
\end{equation}
where $\W[L]\in\fpnset^{\left(\eosalphabetsize+1\right)\times\dimension}$ and $\bb[L]\in\fpnset^{\eosalphabetsize+1}$ are learnable parameters.

We have laid the groundwork that makes the proofs of the following two theorems straightforward.
\tlmtopfo*
\begin{proof}
    The mapping $\lmhead$ consists of a linear layer followed by a $\softmax$.
    Since both the linear layer and $\softmax$ are definable in $\pfo$ as shown in \cref{app:transformer2pfo}, the result follows directly.
\end{proof}

\pototlm*
\begin{proof}
    Let $\automaton = \dfatuple$ be a PODFA, and let $\blm$ be its induced LM.
    By \cref{lemma:po2ptl}, for each state $\stateq \in \states$, there exists a $\ptl$ formula $\tlf_{\stateq}$ that characterizes the automaton being in state $\stateq$.
    Then, by \cref{thm:ptl2transformer}, each such formula can be simulated by a dedicated dimension in the hidden state of a transformer.
    
    Since $\automaton$ is deterministic, at each position, there is exactly one coordinate that takes the value $1$, while the rest are $0$. Therefore, the final prediction head effectively acts as a lookup table that maps each dimension to the target distribution $\localblm$ induced by $\automaton$, which depends solely on the state $\automaton$ is in.
\end{proof}

\section{Experiments}
\label{app:experiments}
In this section, we present additional empirical results that align with our theoretical findings.
\subsection{Language recognition}
We supplement \cref{sec:lr} with detailed descriptions of the experimental setup, tasks, and results.
\subsubsection{Experimental setup}
\label{app:experimental_setup}
Our experimental setup follows \citet{deletang2023neural, butoi2024trainingneuralnetworksrecognizers}. We use a transformer with soft attention, strict future masking, $\layernumber=5$ layers, model size $\dimension=64$, and NoPE. Training strings are of length up to 40, while test strings range from length 41 to 500. The model is trained for 1,000,000 steps with a batch size of 128. For evaluation, we generate 512 samples per test length.
For comparison, we also train a long short-term memory (LSTM) \citep{Hochreiter97} with a hidden size of $256$.
Each experiment is run with $5$ different random seeds and $3$ learning rates $(1\times 10^{-4},3\times 10^{-4},5\times 10^{-4})$. 

All experiments were conducted on a single GPU with 24 GB of memory, each taking approximately one hour to complete. Our code is adapted from \url{https://github.com/google-deepmind/neural_networks_chomsky_hierarchy}, licensed under the Apache License, Version 2.0.

\subsubsection{Languages}
\label{app:tasks}
We consider five language classes, arranged in a strict inclusion hierarchy---each class is a proper subset of the one preceding it. For each class, we select one or more representative languages.
\paragraph{Counter languages.} Counter languages are languages that can be recognized by counter machines---finite-state automata augmented by a number of counters \citep{Fischer1968CounterMA, DBLP:journals/corr/abs-2004-06866}. 
In our experiments, we choose one of the simplest counter languages: \textsc{cnt} $\syma^n\symb^n$, the set of strings consisting of an arbitrary number of $\syma$s followed by an equal number of $\symb$s. 

\paragraph{Regular languages.} A defining characteristic of non-star-free regular languages is modular counting. For instance, \textsc{parity} $\syma^*(\symb\syma^*\symb\syma^*)^*$, the language of binary strings with an even number of $\symb$s, is one of the simplest instances of this class.  

\paragraph{Star-free languages.} Languages that are definable in $\ltl$. We focus on three common examples that are not definable in $\tl$:
\begin{itemize}[leftmargin=*]
\item \textsc{dyck}-$(1,2)$: $(\syma(\syma\symb)^*\symb)^*$, the Dyck language (well-balanced parentheses) with $1$ pair of parentheses, limited to depth $2$.
\item \textsc{dyck}-$(1,1)$: $(\syma\symb)^*$, the Dyck language with $1$ pair of parentheses, limited to depth $1$. This is also a canonical example of strictly locally testable languages, where string membership depends only on adjacent symbol blocks \citep{mcnaughton1971counter}.
\item \textsc{lt}-2: $\kleene{\alphabet}\syma\symb\kleene{\alphabet}$ with $|\alphabet|>2$, the set of strings containing $\syma\symb$ as a substring (symbols appearing contiguously). This is an example of a locally testable language \citep{mcnaughton1971counter}.
\end{itemize} 

\paragraph{Unambiguous polynomials.} Languages that are definable in $\tl$. We are interested in those not definable in $\ptl$, i.e., right-deterministic but not left-deterministic polynomials. We select two such languages:
\begin{itemize}[leftmargin=*]
\item \textsc{rdp}-1: $\kleene{(\alphabet\setminus\{\symb_0\})}\syma\kleene{(\alphabet\setminus\{\syma, \symb_1\})}$, a simple right-deterministic monomial.
\item \textsc{last}: $\kleene{\alphabet}\symb$, the language of strings ending with $\symb$. It can be seen as the simplest representative of the class.
\end{itemize}

\paragraph{Left-deterministic polynomials.} We now consider languages that are definable in $\ptl$, selecting five examples that serve as contrasts to the previously discussed ones.:
\begin{itemize}[leftmargin=*]
\item \textsc{pt}-2: $\kleene{\alphabet}\syma\kleene{\alphabet}\symb\kleene{\alphabet}$ with $|\alphabet|>2$, the language of strings that contain $\syma\symb$ as a subsequence. Languages of this type are known as piecewise testable languages \citep{10.1007/3-540-07407-4_23}.
\item \textsc{lt}-1: $\kleene{\alphabet}\syma\kleene{\alphabet}$, the set of strings that contain $\syma$ as a substring. This is the simplest case in both locally testable and piecewise testable languages.
\item \textsc{ldp}-1: $\kleene{(\alphabet\setminus\{\syma,\symb_0\})}\syma\kleene{(\alphabet\setminus\{\symb_1\})}$, a left-deterministic monomial symmetrical to \textsc{rdp}-1.
\item Although \textsc{ldp}-1 and \textsc{rdp}-1 are symmetrical, the former can be recognized by a DFA with only 2 states, whereas the latter requires 3 (see \cref{fig:automata}). For a fair comparison, we also include \textsc{ldp}-2 $\kleene{(\alphabet\setminus\{\syma_1,\symb_0\})}\syma_1\kleene{(\alphabet\setminus\{\syma_2,\symb_1\})}\syma_2\kleene{(\alphabet\setminus\{\symb_2\})}$, which also requires 3 states.
\item \textsc{first}: $\symb\kleene{\alphabet}$, the set of strings beginning with $\symb$. 
\end{itemize}

\paragraph{Sample generation.} For each language, we construct both positive and negative samples. Some examples are constructed adversarially to increase the difficulty of the classification task.:
\begin{itemize}
\item \textsc{cnt}: Negative samples contain one fewer $\syma$ or $\symb$.
\item \textsc{parity}: Negative samples contain an odd number of $\symb$s.
\item \textsc{dyck}-$(1,2)$, \textsc{dyck}-$(1,1)$:  One symbol in a positive sample is flipped. Since these languages require even-length strings, we only use even-length inputs.
\item \textsc{lt}-2 and \textsc{lt}-1: Negative samples omit $\syma\symb$ (resp. $\syma$) as a substring. Positive samples are constrained to include exactly one such occurrence.
\item \textsc{rdp}-1, \textsc{ldp}-1, and \textsc{ldp}-2: Negative samples contain a single incorrect symbol.
\item \textsc{last} and \textsc{first}: Negative samples do not end or begin with $\symb$, respectively.
\item \textsc{pt}-2: Negative samples lack the $\syma\symb$ subsequence.
\end{itemize}

\begin{figure*}[t!]
    \centering\small
    \begin{subfigure}[c]{0.3\textwidth}
        \centering
        \begin{tikzpicture}
            \node[state, initial] (q0) [] { $\stateq_0$ }; 
            \node[state] (q1) [below = of q0] { $\stateq_1$ }; 
            \node[state, accepting] (q2) [below = of q1] { $\stateq_2$ }; 
            \draw[transition] (q0) edge[auto, bend left] node{$\syma$} (q1) 
            (q1) edge[auto, bend left] node{$\symb_1$} (q0) 
            (q1) edge[auto] node{$\symb_0$} (q2)
            (q0) edge[auto, loop right] node{$\alphabet\setminus\{\syma,\symb_0\}$} (q0)
            (q1) edge[auto, loop right] node{$\alphabet\setminus\{\symb_0,\symb_1\}$} (q1)
            (q2) edge[auto, loop right] node{$\alphabet\setminus\{\syma,\symb_1\}$} (q2);
        \end{tikzpicture}  
        \caption{\textsc{rdp}-1}
    \end{subfigure}
    \begin{subfigure}[c]{0.3\textwidth}
        \centering
        \begin{tikzpicture}
            \node[state, initial] (q0) [] { $\stateq_0$ }; 
            \node[state, accepting] (q1) [below = of q0] { $\stateq_1$ }; 
            \draw[transition] (q0) edge[auto] node{$\syma_1$} (q1) 
            (q0) edge[auto, loop right] node{$\alphabet\setminus\{\syma_1,\symb_0\}$} (q0)
            (q1) edge[auto, loop right] node{$\alphabet\setminus\{\symb_1\}$} (q1);
        \end{tikzpicture}  
        \caption{\textsc{ldp}-1}
        \label{fig:automata_ldp1}
    \end{subfigure}
    \begin{subfigure}[c]{0.3\textwidth}
        \centering
        \begin{tikzpicture}
            \node[state, initial] (q0) [] { $\stateq_0$ }; 
            \node[state] (q1) [below = of q0] { $\stateq_1$ }; 
            \node[state, accepting] (q2) [below = of q1] { $\stateq_2$ }; 
            \draw[transition] (q0) edge[auto] node{$\syma_1$} (q1) 
            (q1) edge[auto] node{$\syma_2$} (q2)
            (q0) edge[auto, loop right] node{$\alphabet\setminus\{\syma_1,\symb_0\}$} (q0)
            (q1) edge[auto, loop right] node{$\alphabet\setminus\{\syma_2,\symb_1\}$} (q1)
            (q2) edge[auto, loop right] node{$\alphabet\setminus\{\symb_2\}$} (q2);
        \end{tikzpicture}  
        \caption{\textsc{ldp}-2}
    \end{subfigure}
    \caption{DFAs for the given languages. Nodes represent states, and arrows represent transitions. The initial state is indicated by an incoming arrow with no source node, and accepting (final) states are shown with double circles. }
    \label{fig:automata}
\end{figure*}
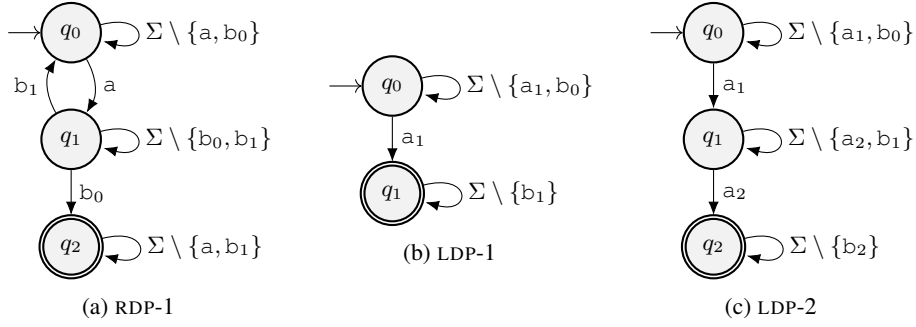

\subsubsection{Results}
\label{app:results}
We compute classification accuracy and report both the maximum and mean values across all runs in \cref{tab:results}. The LSTM achieves perfect accuracy on all tasks, consistent with previous work showing that LSTMs can recognize regular languages \citep{merrill-2019-sequential} and implement counting mechanisms \citep{weiss-etal-2018-practical}. This confirms that the tasks are learnable given the available training data. 

\paragraph{Counter languages.} While several papers have shown that counting quantifiers can be simulated by arbitrary-precision \citep{pmlr-v202-chiang23a,yang2024countingliketransformerscompiling} or log-precision transformers \citep{NEURIPS2023_a48e5877}, our results indicate that a fixed-precision transformer cannot recognize \textsc{cnt}. The model achieves a maximum accuracy of only $83.3\%$. This finding contrasts with \citet{bhattamishra-etal-2020-ability}, who report that transformers can recognize $\syma^n\symb^n\symc^n$. The discrepancy may stem from our evaluation on longer input lengths.

\paragraph{Regular languages.} In line with \citet{hahn-2020-theoretical}, we find that the transformer completely fails to learn the non-star-free regular language \textsc{parity}, reaching at most $52.1\%$ accuracy. Although \citet{chiang-cholak-2022-overcoming} design a transformer with customized positional encodings capable of recognizing \textsc{parity}, these encodings are not always representable under fixed precision, and the resulting architecture is not learnable under standard training conditions.

\paragraph{Star-free languages.} \citet{yang2024masked} show that fixed-precision transformers with UHA recognize exactly the star-free languages. Similarly, \citet{yao-etal-2021-self} demonstrate that transformers with positional encodings of the form $n/\length$---which are not always representable under fixed precision---can recognize bounded Dyck languages, a family of star-free languages. In contrast, we prove that transformers with soft attention and NoPE cannot recognize even the simplest bounded Dyck languages (\textsc{dyck}-$(1,2)$ and \textsc{dyck}-$(1,1)$), nor the locally testable language \textsc{lt}-2. These results are corroborated by our experiments: the transformer fails to learn any of these three languages. The best performance is on \textsc{dyck}-$(1,1)$, with a maximum accuracy of $87.7\%$, which still indicates poor generalization, especially considering the simplicity of the task.

\paragraph{Unambiguous polynomials.} As predicted by our theory, the transformer fails to learn unambiguous polynomials that are not left-deterministic. The model achieves a maximum accuracy of $90.0\%$ on \textsc{rdp}-1 and $64.8\%$ on \textsc{last}.

\paragraph{Left-deterministic polynomials.} Although the transformer cannot learn \textsc{lt}-2, it achieves perfect accuracy on both \textsc{pt}-2 and \textsc{lt}-1. In contrast to the poor performance on \textsc{rdp}-1 and \textsc{last}, the model learns their symmetrical counterparts (\textsc{ldp}-1, \textsc{ldp}-2, and \textsc{first}) with $100\%$ accuracy. Notably, \citet{chiang-cholak-2022-overcoming} construct unmasked transformer encoders capable of recognizing \textsc{first}, but report difficulty in training such models in practice. Our results show that masked transformer decoders can learn \textsc{first} easily and consistently, suggesting that masking may offer a more robust source of positional information than positional encodings.

\paragraph{Summary} The empirical results align fully with our theoretical predictions. Importantly, we use single-precision (32-bit) floating-point numbers, and the string lengths never exceed the maximum attention span of the transformer. That is, attention can uniformly cover all prior positions without numerical underflow or overflow. Yet, despite these favorable conditions, the transformer exhibits no expressive power beyond what is predicted by our formal characterization.

\subsection{Language modeling}
The DFAs corresponding to the languages used in our experiments are shown in \cref{fig:automata}, and their maximum and mean per-token accuracies are reported in \cref{tab:lm}. As predicted, the transformer language model learns left-deterministic polynomials perfectly but fails on the right-deterministic polynomial.
\begin{table}
    \centering
    \caption{Language modeling experiments. Maximum and mean per-token accuracies ($\pm$ standard deviation) are reported. Exact $100.0\%$ are highlighted in bold.}
    \label{tab:lm}
    \begin{tabular}{lcc}
        \toprule
        Language & Max (\%) & Mean (\%) \\
        \midrule
        \textsc{rdp}-1 & 98.3 & $73.6 \pm 12.5$ \\
        \textsc{ldp}-1 & \textbf{100.0} & $98.6 \pm 2.5$ \\
        \textsc{ldp}-2 & \textbf{100.0} & $98.9 \pm 3.7$ \\
        \bottomrule
    \end{tabular}
\end{table}

\section{Additional results}
\label{app:addition}
In this section, we present further theoretical results of potential interest.
\subsection{Hard attention}
\label{app:hard}
Average hard attention (AHA) uniformly attends to positions $m<n$ with the maximum score $\score_{n,m}$. Formally, \cref{eq:softmax_SA} is modified as follows:
\begin{equation}
\label{eq:score_average}
\aalpha_{n,m} \defeq \frac{\ind{\score_{n,m}=\max_{i<n}\score_{n,i}}}{\sum_{j<n} \ind{\score_{n,j}=\max_{i<n}\score_{n,i}}},
\end{equation}
where $\ind{\cdot}$ is the indicator function.

We now show that AHA is also definable in $\pfo$.
\begin{theorem}
\label{thm:aha2pfo}
    Every transformer with AHA can be simulated by $\pfo$.
\end{theorem}
\begin{proof}
    By \cref{prop:max}, $\pfo$ can identify the maximum score $\max_{i < n} \score_{n,i}$.
    The denominator in \cref{eq:score_average} is a sum of non-negative terms and can therefore be simulated by $\pfo$ via \cref{lemma:sum_pos}.
    As with the soft-attention case in \cref{eq:softmax_SA}, the resulting attention weights may vanish when the number of positions exceeds the model's bounded attention span; hence, \cref{lemma:sum_bounded} applies here as well.
    The remaining steps of the construction follow identically to those in \cref{app:transformer2pfo}.
\end{proof}

The other direction is straightforward.
\begin{theorem}
\label{thm:ptl2aha}
    Every $\ptl$ formula can be simulated by a transformer with AHA.
\end{theorem}
\begin{proof}
    Recall that the attention mechanism we constructed in the proof of \cref{lemma:past2transformer} uniformly attends to all positions with the highest score $\score_{n,m}=0$, effectively making it equivalent to AHA. Therefore, $\ptl$ formulas can also be simulated by transformers with AHA.
\end{proof}

Now, we turn to UHA, another widely studied attention mechanism where each position $n$ attends to a single position $m<n$ with the highest $\score_{n,m}$. In the case of ties, the \emph{rightmost} position is selected. \citet{yang2024masked} prove that transformers with UHA are as expressive as $\ltl$. 

Combining the results from above, we have:
\begin{corollary}
    Transformers with AHA are as expressive as those with soft attention, which are strictly less expressive than those with UHA. 
\end{corollary}
\begin{proof}
The equality AHA $=$ soft attention follows directly from \cref{thm:aha2pfo,thm:ptl2aha}.

The strict inequality soft attention $<$ UHA follows directly from the fact that $\ptl$ is strictly less expressive than $\ltl$.
\end{proof}

\subsection{Non-strict future masking}
\label{app:non_strict}
Most commonly used transformers adopt non-strict future-masked soft attention, where each position is allowed to attend to itself. In this case, \cref{eq:softmax_SA} and \cref{eq:attention} become:
\begin{equation}
\aalpha_{n,m} \defeq \frac{\exp(\score_{n,m})}{\sum_{i\leq n} \exp(\score_{n,i})},
\end{equation}
and 
\begin{equation}
    \attention(\transformer(\eosstr))_{:,n} \defeq \sum_{m\leq n} \aalpha_{n,m}\val[:][m].
\end{equation}
respectively. 

This modification, however, limits the expressiveness of the model.
\begin{restatable}{theorem}{nonstrict}
    \label{thm:nonstrict}
    Transformers with non-strict future masking are strictly less expressive than those with strict future masking.
\end{restatable}
\begin{proof}
    Since we have already established that transformers with strict masking are as expressive as $\pfo$, it suffices to show that any transformer with non-strict masking is strictly less expressive than $\pfo$.
    
    We begin by observing that transformers with non-strict masking can be simulated by $\pfo$. The required modification to \cref{prop:count} involves including the current position in the threshold counting quantifiers:
    \begin{subequations}
        \begin{align}
            \exists^{\geq 1} y\leq x: \fof(y) &\defeq \left(\exists^{\geq 1} y<x: \fof(y)\right) \lor \fof(x), \\
            \exists^{\geq 2} y\leq x: \fof(y) &\defeq \left(\exists^{\geq 2} y<x: \fof(y)\right) \lor \left(\left(\exists^{\geq 1} y<x: \fof(y)\right) \land \fof(x)\right), \\
            &\;\;\vdots \nonumber
        \end{align}
    \end{subequations}

    Next, we construct a language that is definable in $\pfo$ but cannot be recognized by a transformer with non-strict masking. Consider the language $\symb\symb\kleene{\alphabet}$, which can be defined in $\pfo$ by the following sentence:
    \begin{equation}
        \exists x: (\atom_\symb(x) \land \left(\exists^{=1} y<x:\atom_\symb(y)\right) \land \neg \exists y<x:\neg\atom_\symb(y)).
    \end{equation}
    
    On the other hand, given any transformer $\transformer$ with non-strict masking, and for any string $\str \in \symb\symb\kleene{\alphabet}$, we have:
    \begin{equation}
        \transformer[\ell][:][2][\eosstr]=\transformer[\ell][:][1][\eosstr] \quad \text{for all layers } \ell\in\{1,\ldots,\layernumber\}
    \end{equation}
    This is because, under non-strict masking, position $n$ attends to all positions $\leq n$, including itself. When the prefix consists entirely of identical symbols, each attention pattern and subsequent computation depend solely on the identical sequence of embeddings, resulting in identical representations at each position.
    However, differentiating the two leading $\symb$s is essential for recognizing $\symb\symb\kleene{\alphabet}$, as the corresponding DFA will enter a different state upon reading the first symbol $\symb$.
\end{proof}

Empirically, we confirm this limitation by training transformers---one with strict masking and one with non-strict masking---on the language $\symb\symb\kleene{\alphabet}$. The results are consistent with our theoretical prediction:
\begin{itemize}
    \item The strictly masked transformer achieves $100.0\%$ maximum accuracy and $96.3\%\pm5.8\%$ average accuracy.
    \item The non-strictly masked variant reaches only $95.8\%$ maximum accuracy and $74.4\%\pm8.2\%$ average accuracy.
\end{itemize}

\subsection{Positional encodings}
\label{app:pe}
Positional encodings are introduced in \citet{NIPS2017_3f5ee243} to inject information about the positions into the transformer. In general, there are two categories of positional encodings:
\begin{itemize}
\item \textbf{Absolute positional encoding:} This is a function $\pe\colon\{1,\ldots,\length+1\}\rightarrow \fpnset^\dimension$. It is typically injected into the input layer by modifying \cref{eq:input} as follows:
\begin{equation}
    \embedding(\eosstr)_{:,n} \defeq \ve(\sym_n) + \pe(n),  \quad n \in\{1,\ldots,\length+1\}.
\end{equation}
\item \textbf{Relative positional encoding:} This is a function $\pe \colon \{1,\ldots,\length+1\}\times\{1,\ldots,\length+1\}\rightarrow \fpnset^\dimension$. It is typically injected into the attention sublayer by modifying \cref{eq:score} as follows \citep{shaw-etal-2018-self, huang2018music}:
\begin{equation}
    \score_{n,m} \defeq \frac{\query[:][n]\bigcdot \key[:][m] + \query[:][n]\bigcdot\pe(n,m)}{\sqrt{\dimension}}.
\end{equation}
\end{itemize}

In formal logic, a numerical predicate is a predicate that depends solely on the positions in a string, not the symbols in those positions. Numerical predicates that depend on one (resp. two) position(s) are referred to as unary (resp. binary) numerical predicates. For instance, the relation $<$ is a binary numerical predicate.

We now show that all absolute (resp. relative) positional encodings can be simulated by unary (resp. binary) numerical predicates.
\begin{restatable}{theorem}{dispe}
Let $\pe$ be an absolute (resp. relative) positional encoding and $\lang\subseteq\kleene{\alphabet}$ be a regular language. There exists a collection of unary (resp. binary) numerical predicates $\mathcal{P}$ such that the following assertions are equivalent:
\begin{enumerate}
    \item $\lang$ can be recognized by a transformer with positional encoding $\pe$.
    \item $\lang$ is definable by $\pfo[\mathcal{P}]$, i.e., $\pfo$ extended with the numerical predicates in $\mathcal{P}$.
\end{enumerate}
\end{restatable}
\begin{proof}
Positional encodings, under fixed precision, have a finite image. Therefore, for every dimension $d \in {1, \ldots, \dimension}$ and every floating-point number $\fpn \in \fpnset$, we can define a numerical predicate $r$ such that:
\begin{equation}
    r(n)=\true \quad \text{ if } \quad \pe(n)_d = \fpn,
\end{equation}
or
\begin{equation}
    r(n,m)=\true \quad \text{ if } \quad \pe(n,m)_d = \fpn.
\end{equation}
\end{proof}

The reverse direction and a precise logical characterization of commonly used positional encodings, e.g., sinusoidal and rotary \citep{SU2024127063}, are left for future work.

\section{Related work}
The expressivity of transformers in the context of formal methods has been extensively studied in recent literature \citep{10.1162/tacl_a_00663}. Various transformer variants have been explored, and different assumptions have been made to enhance the expressivity of transformers.

\paragraph{Chain of thought (CoT).} CoT reasoning, which involves generating intermediate steps before arriving at a final answer, has become a popular approach. \citet{pérez2018on} demonstrated that a transformer with both an encoder and a decoder, when allowed a polynomial number of intermediate computation steps, is Turing complete. Similarly, \citet{merrill2024the, nowak-etal-2024-representational, li2024chain} show that the expressivity of transformers can be improved with various forms of CoT reasoning. In this work, along with many others, we do not allow intermediate steps, which restricts expressivity. In such a case, \citet{hao-etal-2022-formal, merrill-etal-2022-saturated, merrill-sabharwal-2023-parallelism, pmlr-v202-chiang23a, yang2024countingliketransformerscompiling} have established various upper bounds on expressivity.

\paragraph{Non-fixed precision.} If a transformer is allowed arbitrary precision, or if the precision increases with input length, \citet{bhattamishra-etal-2020-ability} proved that it is more expressive than simplified and stateless counter machines. Additionally, \citet{pmlr-v202-chiang23a} and \citet{yang2024countingliketransformerscompiling} showed that a transformer with such precision can recognize any language definable by temporal counting logic. In contrast, our work focuses on the scenario where precision is fixed, which imposes more limitations on expressivity.

\paragraph{Bespoke positional encodings.} Some prior studies have employed bespoke positional encodings, many of which cannot be represented under fixed precision, to overcome certain limitations of transformers. For example, \citet{chiang-cholak-2022-overcoming} used the encoding $n/\length$ to enable transformers to recognize parity, and \citet{barcelo2024logical} demonstrated that transformers with a specially designed positional encoding can achieve a lower bound of $\fo$ with unary numerical predicates. In contrast, our constructions and proofs do not rely on any form of positional encoding.

\paragraph{Hard attention.} Two forms of hard attention have been explored in the literature. \citet{yang2024masked} showed that masked fixed-precision transformers with UHA and NoPE recognize exactly the star-free languages. \citet{barcelo2024logical} established a similar lower bound for transformers with AHA and certain positional encodings. Prior to our work, it was unclear, under fixed precision, whether these two hard attention mechanisms were more or less expressive than, or comparable to, the standard soft attention. Surprisingly, we find that UHA is strictly more expressive than soft attention, while AHA is as expressive as soft attention.\looseness=-1

\section{Limitations}
\label{sec:limit}
The primary limitation of this work lies in the omission of positional encodings. While we briefly discuss their role as numerical predicates, the exact numerical predicates simulated by commonly used positional encodings remain unknown. We hope to explore this in future work. Nevertheless, we believe it is important to first understand the expressivity of a barebones transformer architecture, as this forms the foundation for systematically incorporating various forms of positional encoding later on.

Another limitation---particularly from an empirical perspective---is our stringent evaluation criteria for transformer performance. While certain levels of accuracy might be considered successful in empirical studies, we require perfect accuracy. This is motivated by a formal perspective: a model (e.g., an automaton or logical formula) either recognizes a language or it does not; anything short of perfection is regarded as failure, suggesting that some form of approximation or shortcut has been employed. We interpret our results as identifying a class of tasks that transformers have the full capacity to solve.

\newpage
\section*{NeurIPS Paper Checklist}

The checklist is designed to encourage best practices for responsible machine learning research, addressing issues of reproducibility, transparency, research ethics, and societal impact. Do not remove the checklist: {\bf The papers not including the checklist will be desk rejected.} The checklist should follow the references and follow the (optional) supplemental material.  The checklist does NOT count towards the page
limit. 

Please read the checklist guidelines carefully for information on how to answer these questions. For each question in the checklist:
\begin{itemize}
    \item You should answer \answerYes{}, \answerNo{}, or \answerNA{}.
    \item \answerNA{} means either that the question is Not Applicable for that particular paper or the relevant information is Not Available.
    \item Please provide a short (1–2 sentence) justification right after your answer (even for NA). 
\end{itemize}

{\bf The checklist answers are an integral part of your paper submission.} They are visible to the reviewers, area chairs, senior area chairs, and ethics reviewers. You will be asked to also include it (after eventual revisions) with the final version of your paper, and its final version will be published with the paper.

The reviewers of your paper will be asked to use the checklist as one of the factors in their evaluation. While ``\answerYes{}'' is generally preferable to ``\answerNo{}'', it is perfectly acceptable to answer ``\answerNo{}'' provided a proper justification is given, e.g., ``error bars are not reported because it would be too'' or ``we were unable to find the license for the dataset we used''). In general, answering ``\answerNo{}'' or ``\answerNA{}'' is not grounds for rejection. While the questions are phrased in a binary way, we acknowledge that the true answer is often more nuanced, so please just use your best judgment and write a justification to elaborate. All supporting evidence can appear either in the main paper or the supplemental material, provided in appendix. If you answer \answerYes{} to a question, in the justification please point to the section(s) where related material for the question can be found.

IMPORTANT, please:
\begin{itemize}
    \item {\bf Delete this instruction block, but keep the section heading ``NeurIPS Paper Checklist''},
    \item  {\bf Keep the checklist subsection headings, questions/answers and guidelines below.}
    \item {\bf Do not modify the questions and only use the provided macros for your answers}.
\end{itemize}

\begin{enumerate}

\item {\bf Claims}
    \item[] Question: Do the main claims made in the abstract and introduction accurately reflect the paper's contributions and scope?
    \item[] Answer: \answerYes{} %
    \item[] Justification: Claims made in abstract and \cref{sec:intro} accurately reflect the paper's contributions and scope. 

\item {\bf Limitations}
    \item[] Question: Does the paper discuss the limitations of the work performed by the authors?
    \item[] Answer: \answerYes{} %
    \item[] Justification: \cref{sec:limit}.
    \item[] Guidelines:
    \begin{itemize}
        \item The answer NA means that the paper has no limitation while the answer No means that the paper has limitations, but those are not discussed in the paper. 
        \item The authors are encouraged to create a separate ``Limitations'' section in their paper.
        \item The paper should point out any strong assumptions and how robust the results are to violations of these assumptions (e.g., independence assumptions, noiseless settings, model well-specification, asymptotic approximations only holding locally). The authors should reflect on how these assumptions might be violated in practice and what the implications would be.
        \item The authors should reflect on the scope of the claims made, e.g., if the approach was only tested on a few datasets or with a few runs. In general, empirical results often depend on implicit assumptions, which should be articulated.
        \item The authors should reflect on the factors that influence the performance of the approach. For example, a facial recognition algorithm may perform poorly when image resolution is low or images are taken in low lighting. Or a speech-to-text system might not be used reliably to provide closed captions for online lectures because it fails to handle technical jargon.
        \item The authors should discuss the computational efficiency of the proposed algorithms and how they scale with dataset size.
        \item If applicable, the authors should discuss possible limitations of their approach to address problems of privacy and fairness.
        \item While the authors might fear that complete honesty about limitations might be used by reviewers as grounds for rejection, a worse outcome might be that reviewers discover limitations that aren't acknowledged in the paper. The authors should use their best judgment and recognize that individual actions in favor of transparency play an important role in developing norms that preserve the integrity of the community. Reviewers will be specifically instructed to not penalize honesty concerning limitations.
    \end{itemize}

\item {\bf Theory assumptions and proofs}
    \item[] Question: For each theoretical result, does the paper provide the full set of assumptions and a complete (and correct) proof?
    \item[] Answer: \answerYes{} %
    \item[] Justification: Assumptions: \cref{sec:transformer,sec:tlm,app:transformer}, proofs: \cref{app:transformer2pfo,app:ptl2transformer,app:characterizations,app:tlm,app:addition}.

    \item {\bf Experimental result reproducibility}
    \item[] Question: Does the paper fully disclose all the information needed to reproduce the main experimental results of the paper to the extent that it affects the main claims and/or conclusions of the paper (regardless of whether the code and data are provided or not)?
    \item[] Answer: \answerYes{} %
    \item[] Justification: A description of the experimental setup and data generation is provided in \cref{sec:experiments,app:experimental_setup,app:tasks}.

\item {\bf Open access to data and code}
    \item[] Question: Does the paper provide open access to the data and code, with sufficient instructions to faithfully reproduce the main experimental results, as described in supplemental material?
    \item[] Answer: \answerYes{} %
    \item[] Justification: Code included in the supplementary material.

\item {\bf Experimental setting/details}
    \item[] Question: Does the paper specify all the training and test details (e.g., data splits, hyperparameters, how they were chosen, type of optimizer, etc.) necessary to understand the results?
    \item[] Answer: \answerYes{} %
    \item[] Justification: \cref{sec:experiments,app:experiments}

\item {\bf Experiment statistical significance}
    \item[] Question: Does the paper report error bars suitably and correctly defined or other appropriate information about the statistical significance of the experiments?
    \item[] Answer: \answerYes{} %
    \item[] Justification: Standard deviations are reported in \cref{tab:results,tab:lm,app:non_strict}.
    \item[] Guidelines:
    \begin{itemize}
        \item The answer NA means that the paper does not include experiments.
        \item The authors should answer ``Yes'' if the results are accompanied by error bars, confidence intervals, or statistical significance tests, at least for the experiments that support the main claims of the paper.
        \item The factors of variability that the error bars are capturing should be clearly stated (for example, train/test split, initialization, random drawing of some parameter, or overall run with given experimental conditions).
        \item The method for calculating the error bars should be explained (closed form formula, call to a library function, bootstrap, etc.)
        \item The assumptions made should be given (e.g., Normally distributed errors).
        \item It should be clear whether the error bar is the standard deviation or the standard error of the mean.
        \item It is OK to report 1-sigma error bars, but one should state it. The authors should preferably report a 2-sigma error bar than state that they have a 96\% CI, if the hypothesis of Normality of errors is not verified.
        \item For asymmetric distributions, the authors should be careful not to show in tables or figures symmetric error bars that would yield results that are out of range (e.g. negative error rates).
        \item If error bars are reported in tables or plots, The authors should explain in the text how they were calculated and reference the corresponding figures or tables in the text.
    \end{itemize}

\item {\bf Experiments compute resources}
    \item[] Question: For each experiment, does the paper provide sufficient information on the computer resources (type of compute workers, memory, time of execution) needed to reproduce the experiments?
    \item[] Answer: \answerYes{} %
    \item[] Justification: Computer resources are reported in \cref{app:experimental_setup}.
    
\item {\bf Code of ethics}
    \item[] Question: Does the research conducted in the paper conform, in every respect, with the NeurIPS Code of Ethics \url{https://neurips.cc/public/EthicsGuidelines}?
    \item[] Answer: \answerYes{} %
    \item[] Justification: This research conform with the NeurIPS Code of Ethics. 

\item {\bf Broader impacts}
    \item[] Question: Does the paper discuss both potential positive societal impacts and negative societal impacts of the work performed?
    \item[] Answer: \answerNA{} %
    \item[] Justification: We foresee no societal impact of this work.
    
\item {\bf Safeguards}
    \item[] Question: Does the paper describe safeguards that have been put in place for responsible release of data or models that have a high risk for misuse (e.g., pretrained language models, image generators, or scraped datasets)?
    \item[] Answer: \answerNA{} %
    \item[] Justification: The paper poses no such risks

\item {\bf Licenses for existing assets}
    \item[] Question: Are the creators or original owners of assets (e.g., code, data, models), used in the paper, properly credited and are the license and terms of use explicitly mentioned and properly respected?
    \item[] Answer: \answerYes{} %
    \item[] Justification: Our code is adapted from \citet{deletang2023neural}. The authors are properly credited and the license and terms are explicitly mentioned and properly respected in \cref{app:experimental_setup}. 

\item {\bf New assets}
    \item[] Question: Are new assets introduced in the paper well documented and is the documentation provided alongside the assets?
    \item[] Answer: \answerNA{} %
    \item[] Justification: This paper does not introduce new assets.

\item {\bf Crowdsourcing and research with human subjects}
    \item[] Question: For crowdsourcing experiments and research with human subjects, does the paper include the full text of instructions given to participants and screenshots, if applicable, as well as details about compensation (if any)? 
    \item[] Answer: \answerNA{} %
    \item[] Justification: This research does not involve crowdsourcing.

\item {\bf Institutional review board (IRB) approvals or equivalent for research with human subjects}
    \item[] Question: Does the paper describe potential risks incurred by study participants, whether such risks were disclosed to the subjects, and whether Institutional Review Board (IRB) approvals (or an equivalent approval/review based on the requirements of your country or institution) were obtained?
    \item[] Answer: \answerNA{} %
    \item[] Justification: This research does not involve human subjects.

\item {\bf Declaration of LLM usage}
    \item[] Question: Does the paper describe the usage of LLMs if it is an important, original, or non-standard component of the core methods in this research? Note that if the LLM is used only for writing, editing, or formatting purposes and does not impact the core methodology, scientific rigorousness, or originality of the research, declaration is not required.
    \item[] Answer: \answerNA{} %
    \item[] Justification: This research does not involve LLMs.

\end{enumerate}

\end{document}

%% file: macros.tex
\newcommand{\mymacro}[1]{{#1}}

\newcommand{\defn}[1]{\textbf{#1}}

\newcommand{\paroutline}[3][false]{%
    \ifnum\pdfstrcmp{#1}{true}=0
        #3%
    \else
        [\textit{\textcolor{DiverseMagenta}{#2}}] \textcolor{AccentBlue}{#3}%
    \fi
}

\newcommand{\ind}[1]{\mathbbm{1} \left\{ #1 \right\}}

\newcommand{\B}{{\mymacro{ \mathbb{B}}}}

\newcommand{\equivClass}[1]{{\mymacro{ \left[#1\right]}}}

\newcommand{\alphabet}{{\mymacro{ \Sigma}}}

\newcommand{\eosalphabet}{{\mymacro{ \overline{\alphabet}}}}

\newcommand{\kleene}[1]{{\mymacro{#1^*}}}

\newcommand{\str}{{\mymacro{\textbf{\texttt{w}}}}}

\newcommand{\eosstr}{{\mymacro{\overline{\str}}}}

\newcommand{\strs}{{\mymacro{\textbf{\texttt{s}}}}}
\newcommand{\strt}{{\mymacro{\textbf{\texttt{t}}}}}

\newcommand{\stru}{{\mymacro{\textbf{\texttt{u}}}}}
\newcommand{\strv}{{\mymacro{\textbf{\texttt{v}}}}}

\newcommand{\defeq}{\mathrel{\stackrel{\textnormal{\tiny def}}{=}}}

\newcommand{\setcomplement}[1]{{\mymacro{#1^{\textsf{c}}}}}

\newcommand{\eos}{{\mymacro{\textsc{eos}}}}
\newcommand{\unk}{{\mymacro{\textsc{unk}}}}

\newcommand{\zero}{{\mymacro{\mathbf{0}}}}

\newcommand{\automaton}{{\mymacro{ \mathcal{A}}}}

\newcommand{\stateq}{{\mymacro{ q}}}

\newcommand{\states}{{\mymacro{ Q}}}

\newcommand{\trans}{{\mymacro{ \delta}}}

\newcommand{\initial}{{\mymacro{ I}}}
\newcommand{\final}{{\mymacro{ F}}}

\newcommand{\qinit}{{\mymacro{ q_{\initial}}}}

\newcommand{\dfatuple}{{\mymacro{ \left( \alphabet, \states, \qinit, \final, \trans \right)}}}

\newcommand{\negterm}[1]{{\mymacro{ {\raise.17ex\hbox{$\scriptstyle\sim$}} #1}}}

\newcommand{\ignore}[1]{}
\newcommand{\expandLater}[1]{}

\def\bw{{{\mymacro{ \boldsymbol{\theta}}}}}

\def\1{\mathbf{1}}

\def\rmH{{{\mymacro{ \mathbf{H}}}}}

\def\rmK{{{\mymacro{ \mathbf{K}}}}}

\def\rmQ{{{\mymacro{ \mathbf{Q}}}}}

\def\rmV{{{\mymacro{ \mathbf{V}}}}}

\def\ve{{{\mymacro{ \mathbf{e}}}}}

\def\vp{{{\mymacro{ \mathbf{p}}}}}

\def\vx{{{\mymacro{ \mathbf{x}}}}}
\def\vy{{{\mymacro{ \mathbf{y}}}}}

\def\gL{{{\mymacro{ \mathcal{L}}}}}

\def\gR{{{\mymacro{ \mathcal{R}}}}}

\newcommand{\softmax}{{\mymacro{ \mathrm{softmax}}}}

\newcommand{\ReLU}{{\mymacro{ \mathrm{ReLU}}}}

\newcommand{\bb}[1][]{\ifthenelse{\isempty{#1}}{\mymacro{\mathbf{b}}}{\mymacro{\mathbf{b}^{\text{#1}}}}}
\newcommand{\ff}[1][]{\ifthenelse{\isempty{#1}}{\mymacro{f}}{\mymacro{f_{\text{#1}}}}}

\newcommand{\W}[1][]{\ifthenelse{\isempty{#1}}{\mymacro{\boldsymbol{\Theta}}}{\mymacro{\boldsymbol{\Theta}^{\text{#1}}}}}

\newcommand{\aalpha}{\mymacro{\boldsymbol{\alpha}}}

\newcommand{\stacktop}[1][]{\ifthenelse{\isempty{#1}}{\mymacro{\gamma^{\text{top}}}}{\mymacro{\gamma^{\text{top}}_{#1}}}}
\newcommand{\sym}{\mymacro{\texttt{w}}}
\newcommand{\syma}{\mymacro{\texttt{a}}}
\newcommand{\symb}{\mymacro{\texttt{b}}}
\newcommand{\symc}{\mymacro{\texttt{c}}}

\newcommand{\lang}{\mymacro{\mathbb{L}}}

\newcommand{\onehot}[1][]{\ifthenelse{\isempty{#1}}{\mymacro{ \llbracket \rrbracket}}{\mymacro{ \llbracket#1\rrbracket}}}

\newcommand{\LN}{\mymacro{\text{LN}}}

\newcommand{\embedding}{\mymacro{\mathbf{E}}}
\newcommand{\classification}{\mymacro{\mathrm{C}}}
\newcommand{\lmhead}{\mymacro{\mathbf{L}}}

\newcommand{\qd}{\mymacro{\text{qd}}}
\newcommand{\od}{\mymacro{\text{od}}}

\makeatletter
\newcommand*{\bigcdot}{}%
\DeclareRobustCommand*{\bigcdot}{%
  \mathbin{\mathpalette\bigcdot@{}}%
}
\newcommand*{\bigcdot@scalefactor}{.5}
\newcommand*{\bigcdot@widthfactor}{1.15}
\newcommand*{\bigcdot@}[2]{%
  \sbox0{$#1\vcenter{}$}%
  \sbox2{$#1\cdot\m@th$}%
  \hbox to \bigcdot@widthfactor\wd2{%
    \hfil
    \raise\ht0\hbox{%
      \scalebox{\bigcdot@scalefactor}{%
        \lower\ht0\hbox{$#1\bullet\m@th$}%
      }%
    }%
    \hfil
  }%
}
\makeatother

\newcommand{\true}{\mymacro{\top}}
\newcommand{\TRUE}{\mymacro{\textsc{true}}}
\newcommand{\false}{\mymacro{\bot}}
\newcommand{\FALSE}{\mymacro{\textsc{false}}}

\newcommand{\past}{\mymacro{\mathrel{\textup{\textbf{P}}}}}
\newcommand{\future}{\mymacro{\mathrel{\textup{\textbf{F}}}}}
\newcommand{\since}{\mymacro{\mathrel{\textup{\textbf{S}}}}}
\newcommand{\until}{\mymacro{\mathrel{\textup{\textbf{U}}}}}

\newcommand{\ltl}{\mymacro{\textup{\textbf{LTL}}[\past,\future,\since,\until]}}
\newcommand{\tl}{\mymacro{\textup{\textbf{LTL}}[\past,\future]}}
\newcommand{\ptl}[1][]{\ifthenelse{\isempty{#1}}{\mymacro{\textup{\textbf{LTL}}[\past]}}{\mymacro{\textup{\textbf{LTL}}[\past,#1]}}}

\NewDocumentCommand{\fo}{o}{\IfNoValueTF{#1}{\mymacro{\textup{\textbf{FO}}[\mathord<]}}{\mymacro{\textup{\textbf{FO}}^{#1}[\mathord<]}}}
\NewDocumentCommand{\pfo}{o}{\IfNoValueTF{#1}{\mymacro{\textup{\textbf{PFO}}^2[\mathord<]}}{\mymacro{\textup{\textbf{PFO}}^2[\mathord<,#1]}}}

\newcommand{\dimension}{\mymacro{D}}
\newcommand{\length}{\mymacro{N}}
\newcommand{\layernumber}{\mymacro{L}}
\NewDocumentCommand{\transformer}{O{} O{} O{} o }{
  \IfBlankTF{#1}{
    \IfBlankTF{#2}{
      \IfNoValueTF{#4}{\rmH}{\rmH(#4)}
    }{
      \IfNoValueTF{#4}{\mymacro{\rmH_{#2,#3}}}{\mymacro{\rmH(#4)_{#2,#3}}}
    }
  }{
    \IfBlankTF{#2}{
      \IfNoValueTF{#4}{\mymacro{\rmH^{(#1)}}}{\mymacro{\rmH^{(#1)}(#4)}}
    }{
      \IfNoValueTF{#4}{\mymacro{\rmH^{(#1)}_{#2,#3}}}{\mymacro{\rmH^{(#1)}(#4)_{#2,#3}}}
    }
  }
}
\NewDocumentCommand{\attention}{o}{\IfNoValueTF{#1}{\mymacro{\mathbf{A}}}{\mymacro{\mathbf{A}^{(#1)}}}}
\NewDocumentCommand{\ffn}{o}{\IfNoValueTF{#1}{\mymacro{\mathbf{F}}}{\mymacro{\mathbf{F}^{(#1)}}}}
\newcommand{\cls}{\mymacro{o}}

\newcommand{\fof}{\mymacro{\phi}}
\newcommand{\tlf}{\mymacro{\psi}}
\newcommand{\atom}{\mymacro{\pi}}
\newcommand{\etlf}[1]{\mymacro{\tlf_{<#1}}}
\newcommand{\stlf}[1]{\mymacro{\tlf_{=#1}}}
\newcommand{\ltlf}[1]{\mymacro{\tlf_{#1>}}}

\newcommand{\fpnset}{\mymacro{\mathbb{F}}}
\newcommand{\nnfpnset}{\mymacro{\mathbb{F}_{\geq 0}}}
\newcommand{\fpn}{\mymacro{f}}
\newcommand{\fpnlarge}{\mymacro{f}_{\text{large}}}
\newcommand{\fpnsmall}{\mymacro{f}_{\text{small}}}

\NewDocumentCommand{\query}{o o }{\IfNoValueTF{#1}{\rmQ}{\rmQ_{#1,#2}}}
\NewDocumentCommand{\key}{o o }{\IfNoValueTF{#1}{\rmK}{\rmK_{#1,#2}}}
\NewDocumentCommand{\val}{o o }{\IfNoValueTF{#1}{\rmV}{\rmV_{#1,#2}}}
\newcommand{\score}{\mymacro{\mathbf{S}}}

\newcommand{\equivalent}{\mymacro{\equiv_\lang}}

\newcommand{\ltod}{\mymacro{g}}

\newcommand{\nmax}{\mymacro{N_{\max}}}
\newcommand{\nvalid}{\mymacro{N_{\true}}}

\newcommand{\monoid}{\mymacro{\mathbb{M}}}

\NewDocumentCommand{\ball}{m}{\mymacro{\B_{#1}}}

\newcommand{\pe}{\mymacro{\vp}}

\newcommand{\plm}{\mymacro{p}}
\newcommand{\localplm}{\mymacro{\overrightarrow{p}}}
\newcommand{\preprob}{\mymacro{p}_{\text{prefix}}}
\newcommand{\blm}{\mymacro{p}_{\automaton}}
\newcommand{\localblm}{\mymacro{\overrightarrow{p}_{\automaton}}}

\newcommand{\reject}{\mymacro{\stateq_{R}}}

\newcommand{\eosalphabetsize}{\mymacro{\left|\eosalphabet\right|}}

\newcommand{\mir}[1]{\mymacro{\widehat{#1}}}
\newcommand{\round}[1]{\mathsf{round}\!\left(#1\right)}
\newcommand{\coord}[1]{\mymacro{d_{#1}}}